%% file: neurips_2025.tex
\documentclass{article}

     \usepackage[preprint]{neurips_2025}

\usepackage[utf8]{inputenc} 
\usepackage[T1]{fontenc}    
\usepackage[colorlinks=true,citecolor=blue]{hyperref}       
\usepackage{url}            
\usepackage{booktabs}       
\usepackage{amsfonts}       
\usepackage{nicefrac}       
\usepackage{microtype}      
\usepackage{xcolor}         
\usepackage{amsmath}
\usepackage{amsthm}
\usepackage{algorithm}
\usepackage{algorithmic}
\usepackage{multirow}
\usepackage{makecell}
\usepackage{enumitem}
\usepackage{graphicx}

\input{macro.tex}

\title{Improved Regret in Stochastic Decision-Theoretic Online Learning under Differential Privacy}

\author{%
  Ruihan Wu \\
  UC, San Diego\\
  \texttt{ruw076@ucsd.edu} \\
  \And
  Yu-Xiang Wang \\
  UC, San Diego\\
  \texttt{yuxiangw@ucsd.edu} \\
}

\begin{document}

\maketitle

\begin{abstract}
\citet{hu2024open} posed an open problem: \emph{what is the optimal instance-dependent rate for the stochastic decision-theoretic online learning (with $K$ actions and $T$ rounds) under $\varepsilon$-differential privacy?}
  Before, the best known upper bound and lower bound are $O\left(\frac{\log K}{\Delta_{\min}} + \frac{\log K\log T}{\varepsilon}\right)$ and $\Omega\left(\frac{\log K}{\Delta_{\min}} + \frac{\log K}{\varepsilon}\right)$ (where $\Delta_{\min}$ is the gap between the optimal and the second actions).
  In this paper, we partially address this open problem by having two new results.
  First, we provide an improved upper bound for this problem $O\left(\frac{\log K}{\Delta_{\min}} + \frac{\log^2K}{\varepsilon}\right)$, which is $T$-independent and only has a log dependency in $K$. 
  Second, to further understand the gap, we introduce the \textit{deterministic setting}, a weaker setting of this open problem, where the received loss vector is deterministic.
  At this weaker setting, a direct application of the analysis and algorithms from the original setting still leads to an extra log factor.
  We conduct a novel analysis which proves upper and lower bounds that match at $\Theta(\frac{\log K}{\varepsilon})$.
\end{abstract}

\input{introduction.tex}
\input{preliminary.tex}
\input{main_res.tex}

\bibliographystyle{abbrvnat}  

\bibliography{yourbibfile}

\newpage
\appendix

\input{appendix.tex}


\newpage
\section*{NeurIPS Paper Checklist}

\begin{enumerate}

\item {\bf Claims}
    \item[] Question: Do the main claims made in the abstract and introduction accurately reflect the paper's contributions and scope?
    \item[] Answer: \answerYes{} 
    \item[] Justification: Yes, our main theorem matches the description in the abtract and introduction.

\item {\bf Limitations}
    \item[] Question: Does the paper discuss the limitations of the work performed by the authors?
    \item[] Answer: \answerYes{} 
    \item[] Justification: We discussed the gap of our results towards fully solved open problem.

\item {\bf Theory assumptions and proofs}
    \item[] Question: For each theoretical result, does the paper provide the full set of assumptions and a complete (and correct) proof?
    \item[] Answer: \answerYes{} 
    \item[] Justification: Yes, we have the full proof in the appendix.

    \item {\bf Experimental result reproducibility}
    \item[] Question: Does the paper fully disclose all the information needed to reproduce the main experimental results of the paper to the extent that it affects the main claims and/or conclusions of the paper (regardless of whether the code and data are provided or not)?
    \item[] Answer: \answerNA{} 
    \item[] Justification: This paper is in a theory focus. 

\item {\bf Open access to data and code}
    \item[] Question: Does the paper provide open access to the data and code, with sufficient instructions to faithfully reproduce the main experimental results, as described in supplemental material?
    \item[] Answer: \answerNA{} 
    \item[] Justification: This paper is in a theory focus.

\item {\bf Experimental setting/details}
    \item[] Question: Does the paper specify all the training and test details (e.g., data splits, hyperparameters, how they were chosen, type of optimizer, etc.) necessary to understand the results?
    \item[] Answer: \answerNA{} 
    \item[] Justification: This paper is in a theory focus.

\item {\bf Experiment statistical significance}
    \item[] Question: Does the paper report error bars suitably and correctly defined or other appropriate information about the statistical significance of the experiments?
    \item[] Answer: \answerNA{} 
    \item[] Justification: This paper is in a theory focus.

\item {\bf Experiments compute resources}
    \item[] Question: For each experiment, does the paper provide sufficient information on the computer resources (type of compute workers, memory, time of execution) needed to reproduce the experiments?
    \item[] Answer: \answerNA{} 
    \item[] Justification: This paper is in a theory focus.
    
\item {\bf Code of ethics}
    \item[] Question: Does the research conducted in the paper conform, in every respect, with the NeurIPS Code of Ethics \url{https://neurips.cc/public/EthicsGuidelines}?
    \item[] Answer: \answerYes{} 
    \item[] Justification: Yes we reviewed the NeurIPS Code of Ethics.

\item {\bf Broader impacts}
    \item[] Question: Does the paper discuss both potential positive societal impacts and negative societal impacts of the work performed?
    \item[] Answer: \answerNA{} 
    \item[] Justification: 
    
\item {\bf Safeguards}
    \item[] Question: Does the paper describe safeguards that have been put in place for responsible release of data or models that have a high risk for misuse (e.g., pretrained language models, image generators, or scraped datasets)?
    \item[] Answer: \answerNA{} 
    \item[] Justification:

\item {\bf Licenses for existing assets}
    \item[] Question: Are the creators or original owners of assets (e.g., code, data, models), used in the paper, properly credited and are the license and terms of use explicitly mentioned and properly respected?
    \item[] Answer: \answerNA{} 
    \item[] Justification: 

\item {\bf New assets}
    \item[] Question: Are new assets introduced in the paper well documented and is the documentation provided alongside the assets?
    \item[] Answer: \answerNA{} 
    \item[] Justification: 

\item {\bf Crowdsourcing and research with human subjects}
    \item[] Question: For crowdsourcing experiments and research with human subjects, does the paper include the full text of instructions given to participants and screenshots, if applicable, as well as details about compensation (if any)? 
    \item[] Answer: \answerNA{} 
    \item[] Justification: 

\item {\bf Institutional review board (IRB) approvals or equivalent for research with human subjects}
    \item[] Question: Does the paper describe potential risks incurred by study participants, whether such risks were disclosed to the subjects, and whether Institutional Review Board (IRB) approvals (or an equivalent approval/review based on the requirements of your country or institution) were obtained?
    \item[] Answer: \answerNA{} 
    \item[] Justification: 

\item {\bf Declaration of LLM usage}
    \item[] Question: Does the paper describe the usage of LLMs if it is an important, original, or non-standard component of the core methods in this research? Note that if the LLM is used only for writing, editing, or formatting purposes and does not impact the core methodology, scientific rigorousness, or originality of the research, declaration is not required.
    \item[] Answer: \answerNA{} 
    \item[] Justification: 

\end{enumerate}

\end{document}

%% file: macro.tex
\newcommand{\bbE}{\mathbb{E}}
\newcommand{\bbR}{\mathbb{R}}
\newcommand{\bbP}{\mathbb{P}}
\newcommand{\bbN}{\mathbb{N}}

\newcommand{\calA}{\mathcal{A}}
\newcommand{\calB}{\mathcal{B}}
\newcommand{\calD}{\mathcal{D}}

\newcommand{\calM}{\mathcal{M}}
\newcommand{\calP}{\mathcal{P}}

\newcommand{\calQ}{\mathcal{Q}}

\newtheorem{lemma}{Lemma}
\newtheorem{theorem}{Theorem}
\newtheorem{corollary}{Corollary}
\newtheorem{definition}{Definition}

\newtheorem{remark}{Remark}

%% file: introduction.tex
\section{Introduction}
Differential privacy (DP; \citet{dwork2014algorithmic}) provides a formal guarantee of data privacy, requiring that the outputs from any two datasets differing in a single individual's data do not differ significantly.
In sequential decision-making settings, where the dataset consists of a sequence of observed losses or rewards, DP is extended to compare outputs from two sequences that differ at a single time step.
DP has been extensively studied in two key sequential decision-making frameworks (online learning~\citep{cesa2006prediction,arora2012multiplicative} and multi-arm bandit~\citep{lai1985asymptotically}) under various settings~\citep{jain2012differentially, guha2013nearly,jain2014near,agarwal2017price,tossou2017achieving,sajed2019optimal,hu2021near,asi2023private}.

In this paper, we study stochastic decision-theoretic online learning~\citep{freund1997decision} under \emph{pure} differential privacy, a setting identified as an \emph{open problem} by \citet{hu2024open}.
In this problem, there are $K$ actions, each associated with an unknown distribution of loss.
At each time step, the learner selects an action and observes a stochastic loss drawn from its distribution.
The problem is under the full-information setting, where the learner observes the stochastic losses of all actions after picking one, with the goal of minimizing the expected cumulative loss over time.

\begin{table}[!t]
\label{tab:result}
\centering
\caption{A summarization of the previous existing results and our new results for the problem \textit{stochastic decision-theoretic online learning under differential privacy}.}
\resizebox{\linewidth}{!}{
\begin{tabular}{c|c|c}
	\toprule
	Settings 	& Lower bound & Upper bound\\
	\midrule
	\multirow{2}{*}{\makecell{Instance-dependent bound\\ for the \emph{original setting}}} & \multirow{2}{*}{$\frac{\log K}{\Delta_{\min}} + \frac{\log K}{\varepsilon}$~\citep{hu2021near}} & $\frac{\log K}{\Delta_{\min}} + \frac{\log K \log T}{\varepsilon}$~\citep{hu2021near}\\
	\cmidrule{3-3}
	& & $\frac{\log K}{\Delta_{\min}} + \frac{\log^2 K}{\varepsilon}$ (\textbf{This work})\\
	\midrule
	\multirow{3}{*}{\makecell{Instance-\emph{in}dependent bound\\ for the \emph{original setting}}} & \multirow{3}{*}{$\sqrt{T\log K} + \frac{\log K}{\varepsilon}$~\citep{hu2021near}} & $\sqrt{T\log K} + \frac{K\log K\log^2 T}{\varepsilon}$~\citep{jain2014near}\\
	\cmidrule{3-3} 
	& & $\sqrt{T\log K} + \frac{\log K \log T}{\varepsilon}$~\citep{asi2023private, hu2021near}\\
	\cmidrule{3-3} 
	& & $\sqrt{T\log K} + \frac{\log^2 K}{\varepsilon}$ (\textbf{This work})\\
	\midrule
	\multirow{2}{*}{\makecell{The \textit{deterministic setting}}} & \multirow{2}{*}{$\frac{\log K}{\varepsilon}$ (\textbf{This work})} & $\frac{\log^2 K }{\varepsilon}$ (extended from our result in the original setting)\\
	\cmidrule{3-3} 
	& & $\frac{\log K}{\varepsilon}$ (\textbf{This work})\\
	\bottomrule
\end{tabular}
}
\end{table}

\begin{table}[!t]
\centering
\caption{Detailed specifications in Algorithm~\ref{alg:main} that achieve the exising result and our new results.}
\label{tab:spec}
\resizebox{\linewidth}{!}{
\begin{tabular}{c|c|c}
	\toprule
		& Bernoulli resampling or not ($B$) & Noise distribution in report-noisy-max ($\calQ_\varepsilon$) \\
	\midrule
	Theorem~\ref{thm:eixst} \citep{hu2021near} & no & Laplace distribution\\
	\midrule
	\makecell{Theorem~\ref{thm:main} (\textbf{This work})} & yes & Laplace distribution, Exponential distribution, Gumbel distribution\\
	\midrule
	\makecell{Theorem~\ref{thm:upper_det} (\textbf{This work})\\ (for \textit{deterministic setting})} & no & Exponential distribution, Gumbel distribution\\
	\bottomrule
\end{tabular}
}
\end{table}

\citet{jain2014near} provides an instance-independent bound $O\left( \sqrt{T\log K}  + \frac{K\log K\log^2 T}{\varepsilon}\right)$ for general online linear optimization, which can be adapted as an upper bound for this problem.
The best instance-independent bound so far for this problem is $O\left(\sqrt{T\log K} + \frac{\log K \log T}{\varepsilon}\right)$, achieved by \citet{asi2023private} and \citet{hu2021near}, where the lower bound is $O\left(\sqrt{T\log K}+\frac{\log K}{\varepsilon}\right)$.
Particularly, the open problem~\citep{hu2024open} asked for the instance-dependent bound in terms of $K, T, \varepsilon, \Delta_{\rm min}$, where $\Delta_{\min}$ is the gap of expected losses between the optimal and the second actions. 
The best existing instance-dependent bound is $O\left(\frac{\log K}{\Delta_{\min}} + \frac{\log K \log T}{\varepsilon}\right)$~\citep{hu2021near} and the proved lower bound is $\Omega\left(\frac{\log K}{\Delta_{\min}} + \frac{\log K}{\varepsilon}\right)$.
The algorithm in \citet{hu2021near} for these two bounds is quite standard: the algorithm applies a doubling metric to divide the time dimensions into epochs. 
At each epoch, it accumulates the observed loss vectors first, and uses a standard DP mechanism, report-noisy-max~\citep{dwork2014algorithmic} with Laplace noise, to pick an action for the whole next epoch.
The algorithm is presented in Algorithm~\ref{alg:main}.

We propose a variant of the algorithm from \citet{hu2021near} that first resamples the stochastic loss vectors into Bernoulli variables before accumulating them.
This modification enables an analysis of a new instance-dependent upper bound of $O\left(\frac{\log K}{\Delta_{\min}} + \frac{\log^2 K}{\varepsilon}\right)$.
Compared to the best existing bound stated in Theorem~\ref{thm:eixst}, our result improves performance when $T>K$ (a small burn-in period).
Notably, we are the first to demonstrate that the instance-dependent regret remains constant in $T$ and depends only logarithmically on $K$, as the lower bound predicts.
As a simple corollary, it also provides a new instance-independent upper bound $O\left(\sqrt{T\log K} + \frac{\log^2 K}{\varepsilon}\right)$ with a similar improvement.

By comparing the upper and lower bound, the extra factor appears alongside the DP parameter $\varepsilon$.
To better understand the existence of the extra factor, 
we consider a strictly weaker setting in which the received loss vectors are deterministic, allowing us to isolate differential privacy (DP) from the stochasticity of observed losses.
Although this setting is not realistic, directly applying the analysis and algorithms developed for the original open problem still yields the same extra factor.
In this weaker setting, we introduce a new variant of the algorithm from \citet{hu2021near}, replacing the Laplace noise in the report-noisy-max mechanism with either exponential or Gumbel noise.
We prove a lower bound for this deterministic setting and show that the upper bound achieved by our new algorithm matches it, attaining a rate of $\frac{\log K}{\varepsilon}$.
Finally, we discuss how these findings in the simplified setting offer insights into the original open problem.

We summarize both prior results and our new contributions in Table~\ref{tab:result}.
Additionally, Table~\ref{tab:spec} outlines the specific variants of Algorithm~\ref{alg:main} that achieve the existing and proposed results.
\textbf{The organization of this paper}: 
Section~\ref{sec:prelim} introduces the problem setting, the existing results, and the related work;
Section~\ref{sec:main} presents our new results for the open problem;
Section~\ref{sec:main_det} describes our new results for the deterministic setting of the open problem and discusses how these findings may inform further progress;
The appendix contains additional proofs and technical details.

\begin{algorithm}[!t]
\caption{Variants of RNM-FTNL($B$, $\calQ_{\varepsilon}$)}
\label{alg:main}
\begin{algorithmic}[1]
\STATE \textbf{Specifying the variant:} a bit $B\in\{0, 1\}$ for indicating whether the loss vector is resampled or not; a noise distribution $\calQ_{\varepsilon}$ parametrized by $\varepsilon$.\ \ \ \ \texttt{\small $\backslash\backslash$ The original RNM-FTNL~\citep{hu2021near} can be recovered by setting $B=0$ and $\calQ_{\varepsilon}$ as the laplace distribution $\mathrm{Lap}(\frac{2}{\varepsilon})$.} 
\STATE \textbf{Input:} Action set $[K]$ and privacy parameter $\varepsilon$
\STATE Draw $J_{0}$ from a uniform distribution over $[K]$.
\FOR{$r=1, \cdots, \lceil\log_2(T-1)\rceil+1$}
\STATE Set $G_{r} = (0, \cdots, 0)\in \bbR^{K}$
\FOR{$t=2^{r-1}, \cdots, 2^{r}-1$}
\STATE Play the action $I_t \leftarrow J_{r-1}$.
\STATE Receive the loss vector $\ell^{(t)}=(\ell^{(t)}_1, \cdots, \ell^{(t)}_K)\sim \calP_1\times \cdots \times \calP_K$.
\IF{$B=0$}
\STATE $\tilde{\ell}^{(t)} \leftarrow \ell^{(t)}$
\ELSE
\STATE $\tilde{\ell}^{(t)} \leftarrow (\tilde{\ell}_{1}^{(t)}, \cdots, \tilde{\ell}_{K}^{(t)})\sim \calB(\ell_{1}^{(t)}) \times\cdots \times \calB(\ell_{K}^{(t)})$, where $\calB(p)$ is the Bernoulli distribution with mean $p$.\ \ \ \ \texttt{\small $\backslash\backslash$ Bernoulli resampling}
\ENDIF
\STATE $G_{r}\leftarrow G_{r} + \tilde{\ell}^{(t)}$
\ENDFOR
\STATE $J_{r} \leftarrow \arg\max_{j\in K} -G_{r, j}  + Q_{r, j}$ where $Q_{r, j}\sim \calQ_{\varepsilon}$
\ENDFOR
\end{algorithmic}
\end{algorithm}

%% file: preliminary.tex
\section{Preliminaries}
\label{sec:prelim}
\subsection{Problem setting}
In this paper, we focus on the open problem posed by \citet{hu2024open} and we begin by reviewing the problem setting in this section.
The stochastic variant of decision-theoretic online learning~\citep{freund1997decision} assumes a finite set of $K$ actions.
Each action $i\in [K]$ is associated with an unknown loss distribution $\calP_i$ that is unknown to the learner, whose support lies within $[0,1]$ --  this support assumption follows the same set-up as the open problem in \citet{hu2024open}.

At each time step $t=1, \cdots, T$:
\begin{enumerate}[leftmargin=*,nosep]
\item The learner picks any action $I_t\in [K]$ according to any (randomized) algorithm $\calM$.
\item The learning algorithm suffers loss $\ell_{I_t}^{(t)}\sim \calP_{I_t}$.
\item The learner observes the losses of all the actions, a loss vector $\ell^{(t)}:=(\ell^{(t)}_1, \cdots, \ell^{(t)}_K)\sim \calP_1\times \cdots \times \calP_K$.
\end{enumerate}
The goal is to minimize the pseudoregret $\mathrm{PseudoRegret}(\calA; T, \calP_1, \cdots, \calP_K)$, which is the gap between the expectation of accumulated suffered losses and the minimum expectation of accumulated loss among $K$ actions:
$$
\bbE\left[\sum_{t=1}^T \ell_{I_t}^{(t)}\right] - \min_{i\in [K]}\bbE\left[\sum_{t=1}^T \ell_{i}^{(t)}\right],
$$
where the randomness in the expectation is contributed by both the loss vector $\ell^{(t)}$ and the randomized algorithm $\calM$.
We further denote $\mu_i$ as the expectation of the loss from action $i$, $\bbE_{\ell_i\in \calP_i}\left[\ell_{i}\right]$.
Without the loss of generality, we assume $\mu^* = \mu_1 < \mu_2\leq\cdots \mu_K$.
Furthermore, we denote the gaps $\Delta_i:=\mu_i - \mu_1$ and specifically, we denote the gap between the optimal and second optimal by $\Delta_{\min}:= \mu_2 - \mu_1$.
With the notations of gaps, the pseudoregret can be rewritten:
\begin{align}
\label{eq:regret}
\mathrm{PseudoRegret}(\calA; T, \calP_1, \cdots, \calP_K) &= \bbE\left[\sum_{t=1}^T \mu_{I_t}\right] - T\cdot\mu_1 = \sum_{t=1}^T \bbE\left[\Delta_{I_t}\right].
\end{align}
The optimal rate for the pseudoregret at this non-private setting is $\frac{\log(K)}{\Delta_{\min}}$, given by \citet{kotlowski2018minimaxity,mourtada2019optimality}.

In this paper, we study the problem under the framework of differential privacy (DP; \citet{dwork2014algorithmic}), a standard definition of privacy that requires the outcome distribution from the given randomized algorithm would not be changed too much if only one individual in the dataset has been changed.
Particularly, differential privacy in online learning~\citep{dwork2010differential} is \textit{event-level},  which assumes the individual is the loss vector at a single time step $t$ and the formal definition is as follow; also in this paper we only consider the \emph{pure} DP rather than \textit{approximate} DP, as set in the open prolem~\citep{hu2024open}.

\begin{definition}[Differential privacy in online learning]
	A randomized online learning algorithm $\calM$ is $\varepsilon$-differentially private if for any two loss vector sequences $\ell^{(1:t)}=(\ell^{(\tau)})_{\tau \in[t]}$ and $(\ell')^{(1:t)}$ differing in at most one vector and any decision set $\calD_{1:t}\subseteq [K]^t$, we have $\bbP[\calM(\ell^{(1:t)})\in\calD_{1:t}]\leq e^{\varepsilon}\cdot\bbP[\calM((\ell')^{(1:t)})\in\calD_{1:t}]$ for all $t\leq T$.
\end{definition}

We now state the open problem posed by \citet{hu2024open}: for the stochastic variant of decision-theoretic online learning,
$$\textbf{what is the optimal instance-dependent rate for the pseudoregret under $\varepsilon$-differential privacy?}$$
Or equivalently, what is the optimal rate in terms of $\varepsilon, \Delta_{\min}, K, T$ for the pseudoregret (Equation~\ref{eq:regret}) that can be achieved by any algorithm?


\subsection{Best existing results}
\label{sec:exist}
The best lower bound for this open problem so far, proved by \citet{hu2021near}, is 
\begin{equation}
\label{eq:low}
\Omega\left(\frac{\log K}{\Delta_{\min}} + \frac{\log K}{\varepsilon}\right).
\end{equation}
The lower bound means that the pseudoregret of any $\varepsilon$-DP algorithm cannot have a better rate than this lower bound for all problem instances ($T, \calP_1, \cdots, \calP_k$).
\citet{hu2021near} also introduces the algorithm FNM-FTNL, which achieves the best rate so far for upper bounding the pseudoregret, 
$$O\left(\frac{\log K}{\Delta_{\min}} + \frac{\log K \log T}{\varepsilon}\right).$$
We present their algorithm FNM-FTNL in Algorithm~\ref{alg:main}, by specifying $B=0$ and the noise distribution $\calQ_{\varepsilon}$ as the laplace distribution $\mathrm{Lap}(\frac{2}{\varepsilon})$; $\mathrm{Lap}(\beta)$ has the probability density function $f(x)=\frac{\beta}{2}e^{-\frac{|x|}{\beta}}$ for $x\in \mathbb{R}$.
The algorithm applies a doubling trick to divide the time dimensions into epochs. 
At each epoch $r$, it accumulates the received loss vectors first and uses the report-noisy-max DP mechanism~\citep{dwork2014algorithmic} (with the laplace noise) to pick an action $J_r$ for the next epoch $r+1$ while preserving the $\varepsilon$-DP guarantee.
We formally state their results in the following theorem.
\begin{theorem}[Best existing result; \citep{hu2021near}.]
\label{thm:eixst} 
	When specifying $B=0$ and $\calQ_{\varepsilon}$ as the laplace distribution $\mathrm{Lap}(\frac{2}{\varepsilon})$, Algorithm~\ref{alg:main} is $\varepsilon$-differentially private and satisfies the gaurantee
	\begin{equation}
	\label{eq:pseudo_regret_exist}
	\mathrm{PseudoRegret}(\text{RNM-FTNL}(B, \calQ_{\varepsilon}); T, \calP_1, \cdots, \calP_K) = O\left(\frac{\log K}{\Delta_{\min}} + \frac{\log K \log T}{\varepsilon}\right).
	\end{equation}
\end{theorem}

\subsection{Related work}
The open problem is considering one specific setting in private online prediction from experts~\citep{asi2023private}.
Private online prediction from expert advice can have bandit setting and full information setting~\citep{guha2013nearly}, based on assuming the learner observes the reward or loss only from the selected action at the time or from all actions.
There are three models of adversaries at the full information setting from the strongest to the weakest: \textit{adaptive} adversaries, who can decide the loss (distribution) upon from the picked action from the last time step~\citep{jain2012differentially, guha2013nearly, jain2014near, agarwal2017price, asi2023private}; \textit{oblivious} adversaries, who decide a sequence of loss distributions before the online procedure~\citep{asi2023private}; \textit{stochastic} adversaries, who pick one loss distribution and at each time step sample the loss i.i.d. from this distribution~\citep{kairouz2021practical, hu2021near, asi2023private}.
The open problem studied in this paper is at the full information setting with the stochastic adversary.

Private online prediction from experts is a special case of private online linear optimization (OLO) and private online convex optimization (OCO)~\citep{guha2013nearly, agarwal2017price, kairouz2021practical, agarwal2023differentially,asi2023near, pmlr-v235-agarwal24d}, where the optimization constraint is as an L1-sphere.
Private OLO has been studied with different constraints too, such as the L2-ball or the cube, at both full-information setting and bandit setting.

%


%% file: main_res.tex
\section{Main Result for the Open Problem}
\label{sec:main}

\subsection{An improved $T$-independent upper bound for the open problem.}
\label{sec:result_original}
\paragraph{New algorithm with the \textit{Bernoulli resampling}.} 
As introduced in Section~\ref{sec:exist}, the original RNM-FTNL algorithm achieves the best known upper bound.
Our improved regret rate is obtained by modifying RNM-FTNL with a key enhancement.
While retaining the same doubling trick and accumulation of loss vectors, we introduce an additional step called \textit{Bernoulli resampling}: each observed loss vector is resampled through a joint distribution of Bernoulli random variables, preserving the original expectation for each coordinate.
That is, each coordinate in the resampled vector is a Bernoulli variable with mean equal to its corresponding observed loss.
This modified algorithm is presented in Algorithm~\ref{alg:main}, with the resampling step specified by setting $B=1$.

First of all, notice that Bernoulli resampling reduces all problem instances to a proper subset of instances with Bernoulli loss distributions.
This means that the worst case of our new algorithm is \emph{not worse} than the one of the original RNM-FTNL\footnote{It remains unknown if Bernoulli resampling strictly improves the worst case.}.
Moreover, the step of Bernoulli resampling strictly reduces the regret for some problem instances.
Here is one example of such problem instance: Suppose there are two actions and their loss probabilities respectively are
$$
\bbP[\ell_1=0.3]=1; \bbP[\ell_2=0.4]=0.8, \bbP[\ell_2=0]=0.2;  
$$
In this example, action $1$ is the optimal action with the lower expected loss. 
However, the action $J_1$ to take decided by the original RNM-FTNL (see line 16 in Algorithm~\ref{alg:main}) will weigh more on the suboptimal action $2$: at an extreme case when there is no DP guarantee ($\varepsilon=\infty$), $P[J_1=2]=0.8$.
With the Bernoulli resampling in our new algorithm, a lemma introduced later states that the suboptimal action $2$ will always have less chance than the optimal action $1$ to be selected.
This means that our algorithm will force $P[J_1=2]\leq 0.5\leq P[J_1=1]$ to happen and as consequence, our algorithm will suffer smaller loss from this action.

\paragraph{Extending the algorithm by different DP mechanisms.} 
In the original algorithm, the Laplace noise is used to satisfy DP. Our main result also shows that the same results will hold for other two DP mechanisms: adding Exponential noise or the Gumbel noise (known as Exponential mechanism). Denote the Exponential distribution by $\mathrm{Exp}(\beta)$ with the probability density function $f(x)=\frac{1}{\beta} e^{-\frac{x}{\beta}}$ for $x\geq 0$, and denote the Gumbel distribution by $\mathrm{Gumbel}(\beta)$ with the probability density function $f(x)=\frac{1}{\beta}e^{-\frac{x}{\beta} - e^{-\frac{x}{\beta}}}$ for $x\in \mathbb{R}$. When $\calQ_{\varepsilon}$ is specified as $\mathrm{Exp}(\frac{1}{\varepsilon})$ and $\mathrm{Gumbel}(\frac{2}{\varepsilon})$, similar to the analysis for $\calQ_{\varepsilon}=\mathrm{Lap}(\frac{2}{\varepsilon})$, it is also proved that Algorithm~\ref{alg:main} is $\varepsilon$-DP. 
This is because each $J_r$ is $\varepsilon$-DP w.r.t. the received loss vectors in the last epoch~\citep{dwork2014algorithmic,qiao2021oneshot} and the sets of loss vectors in each epoch are disjoint.
 
In a later section, we will show that using these two alternative noise distributions allows for a tighter analysis in a strictly weaker setting—an analysis that we are not able to conduct with the Laplace noise used in the original RNM-FTNL.
This raises the possibility that these alternative noise distributions could also enable tighter analysis for the original open problem.
For completeness, we will demonstrate that these two additional DP mechanisms achieve the same regret bounds as those obtained using Laplace noise.

\paragraph{Main result: new $T$-independent rate for the open problem.}
In summary, our main result will be built on our new algorithm, which extends the original FNM-FTNL by adding the Bernoulli resampling and other DP mechanisms. The exact specifications of $B$ and $\calQ_{\varepsilon}$ in Algorithm~\ref{alg:main} for attaining our main result are listed in Table~\ref{tab:spec}.
We now formally state our main result and the proof idea will be introduced in the next subsection.

\begin{theorem}[Main result: new rate for the open problem.]
\label{thm:main}
	When specifying $B=1$ and $\calQ_{\varepsilon}$ as the Laplace distribution $\mathrm{Lap}(\frac{2}{\varepsilon})$, the Exponential distribution $\mathrm{Exp}(\frac{1}{\varepsilon})$, or the Gumbel distribution $\mathrm{Gumbel}(\frac{2}{\varepsilon})$, Algorithm~\ref{alg:main} is $\varepsilon$-differentially private and satisfies the guarantee
	\begin{equation}
		\label{eq:regret_main}
	\mathrm{PseudoRegret}(\text{RNM-FTNL}(B, \calQ_{\varepsilon}); T, \calP_1, \cdots, \calP_K) = O\left(\frac{\log K}{\Delta_{\min}} + \frac{\log^2K}{\varepsilon}\right).
	\end{equation}
\end{theorem}
\begin{remark}
First, by comparing our new upper bound with the best existing upper bound in Theorem~\ref{thm:eixst}, it improves over existing results when $T>K$ (a small burn-in period).
Notably, we are the first to show that the instant-dependent regret remains constant in $T$ and is only $\log K$ dependent, as the lower bound (Equation~\ref{eq:low}) predicts.
 From the analysis of \citet{hu2021near}, we observe that it was possible to have a upper bound of $O\left(\frac{\log K}{\Delta_{\min}} + \frac{K\log K}{\varepsilon}\right)$ by bounding the number of groups in $K$, which is also $T$-independent.
 However, our bound is much tighter than their analysis, reduces the $T$-independent gap from $K$ to $\log K$.
\end{remark}


The instance-dependent bound from Theorem~\ref{thm:main} further imply a new instance-independent bound.
We present the result in the following theorem; the proof follows the same steps as a similar corollary in~\citet{hu2021near} and we put the proof in Appendix~\ref{sec:app_proof_inst_ind}.
\begin{corollary}[New rate for the instance-independent bound]
\label{cor:inst_ind}
	When specifying $B=1$ and $\calQ_{\varepsilon}$ as the Laplace distribution $\mathrm{Lap}(\frac{2}{\varepsilon})$, the Exponential distribution $\mathrm{Exp}(\frac{1}{\varepsilon})$, or the Gumbel distribution $\mathrm{Gumbel}(\frac{2}{\varepsilon})$, Algorithm~\ref{alg:main} is $\varepsilon$-differentially private and satisfies the gaurantee
	$$
	\mathrm{PseudoRegret}(\text{RNM-FTNL}(B, \calQ_{\varepsilon}); T, \calP_1, \cdots, \calP_K) = O\left(\sqrt{T\log K} + \frac{\log^2 K}{\varepsilon}\right).
	$$
\end{corollary}

\subsection{Proof of Theorem~\ref{thm:main}}
\paragraph{A lemma given by Bernoulli resampling.} We first introduce an important property given by the \textit{Bernoulli resampling}.
For any $j_1<j_2$, i.e. the action $j_1$ has the smaller loss in expectation, suppose $J_r$ is the output of report-noisy-max mechanism (line 16 in Algorithm~\ref{alg:main}) and the Bernoulli resampling will enforce the truth of $\bbP[J_r=j_1]\geq \bbP[J_r=j_2]$.
In simple words, the better action always has larger chance to be selected as the action for the next epoch $r+1$.
Moreover, since $\bbP[J_r=j]\leq \bbP[J_r=j-1] \leq\cdots\leq \bbP[J_r=1]$, we have the upper bound $\bbP[J_r=j]\leq \frac{1}{j}$, which helps us derive a more fine-grained analysis.
We formally state this conclusion in the next lemma.

\begin{lemma}[Monotonicity for bionomial distributions]
\label{lem:monotonicity}
	Suppose $J_r$ is the output from report-noisy-max, as defined at line 16 in Algorithm~\ref{alg:main}. When we specify Algorithm~\ref{alg:main} by $B=1$ and the noise distribution $\calQ_{\varepsilon}$ is $\mathrm{Lap}(\frac{2}{\varepsilon})$, $\mathrm{Exp}(\frac{1}{\varepsilon})$ or $\mathrm{Gumbel}(\frac{2}{\varepsilon})$. For any $r\geq 1$ and $j_1 < j_2$, $\bbP[J_r = j_1]\geq \bbP[J_r = j_2]$. Moreover, $\bbP[J_r = j]\leq \frac{1}{j}$.
\end{lemma}

The proof follows the steps
\begin{enumerate}[leftmargin=*,nosep]
\item With the Bernoulli resampling, the accumulated loss $G_{r, j}$ has the binomial distribution $\calB(2^{r-1}, \mu_j)$. Binomial distribution has the property~(\citet{wadsworth1960introduction}; Appendix~\ref{sec:app_bern})
	$$\mu_{j_1}\leq\mu_{j_2}\Rightarrow \forall x, F_{G_{r, j_1}}(x) \geq F_{G_{r, j_2}}(x).$$
	where $F_{A}(x)$ is denoted as the cumulative density function for any random variable $\bbP[A\leq x]$.
\item Denote $N_{r, j} := -G_{r, j}+Q_{r, j}$. Because $Q_{r, j}$ share the same distribution $\calQ_{\varepsilon}$ and $ F_{G_{r, j_1}}(x) \geq F_{G_{r, j_2}}(x)$ from the step 1, we can prove that $F_{N_{r, j_1}}(x) \leq F_{N_{r, j_2}}(x)$.
\item Let $H = \max_{j\neq j_1, j_2}N_{r, j}$. By applying $F_{N_{r, j_1}}(x) \leq F_{N_{r, j_2}}(x)$ from the step 2, we can prove 
$$
\bbP[J_r = j_1] = \bbP[N_{r, j_1}> \max\{N_{r, j_2}, H\}] \geq \bbP[N_{r, j_2}> \max\{N_{r, j_1}, H\}] = \bbP[J_r = j_2].
$$
\end{enumerate}
We put the full proof in the Appendix~\ref{sec:app_monotonicity}.

\paragraph{Proof sketch of Theorem~\ref{thm:main}.} Now we show the proof sketch for Theorem~\ref{thm:main} by omitting some calculations that are similar to the proof in \citet{hu2024open}; the complete proof is in Appendix~\ref{sec:app_proof_main}.
\begin{proof}[Proof sketch of Theorem~\ref{thm:main}.]
The Algorithm~\ref{thm:main} is $\varepsilon$-differentially private as discussed at the beginning of this section. Next, we are going to bound the pseudoregret.
If we can prove Equation~\ref{eq:regret_main} for any $T:=2^{R}-1$ where $R$ is any non-negative integer, Equation~\ref{eq:regret_main} would also hold for arbitrary $T$, because Algorithm~\ref{alg:main} is independent of the $T$ and the regret of Algorithm~\ref{alg:main} is non-decreasing in $T$.
Therefore, we can assume $T:=2^{R+1}-1$ for some non-negative integer $R$ and can rewrite the pseudoregret (defined in Eqeation~\ref{eq:regret}) according to the Algorithm~\ref{alg:main}:
$$
\sum_{t=1}^T \bbE\left[\Delta_{I_t}\right] = \sum_{r=1}^{R} \sum_{t=2^{r-1}}^{2^{r}-1}\sum_{j=1}^{K}\Delta_{j}\bbP[J_{r-1} = j] = \sum_{j=1}^{K}\Delta_{j}\sum_{r=1}^{R} 2^{r-1}\bbP[J_{r-1} = j]
$$
$$ 
 = \underbrace{\sum_{j:\Delta_j\leq \varepsilon}\Delta_{j}\sum_{r=1}^{R} 2^{r-1}\bbP[J_{r-1} = j]}_{\text{Regret}_{\uparrow}} + \underbrace{\sum_{j:\Delta_j > \varepsilon}^{K}\Delta_{j}\sum_{r=1}^{R} 2^{r-1}\bbP[J_{r-1} = j]}_{\text{Regret}_{\downarrow}}
$$

$\text{Regret}_{\uparrow}\leq O\left(\frac{\log K}{\Delta_{\min}}\right)$ can be shown in the same way as a part of proof for Theorem 9 in ~\citet{hu2021near}.
The remaining is to focus on bounding $\text{Regret}_{\downarrow}$, where Lemma~\ref{lem:monotonicity} is applied.

According to the Lemma 9 in \citet{hu2021near} and  Lemma~\ref{lem:more_dist} (in Appendix), for all three noise distributions $\calQ_{\varepsilon}$, there exists universal constants $c_1, c_2>0$ s.t.
	\begin{align}
	\label{eq:tail_main}
	\bbP\left[J_r = j\right]\leq c_1\cdot \exp(-2^{r+1}\Delta_j \min\{\Delta_j, \varepsilon\} / c_2).
	\end{align}
With this property and a similar proof for Theorem 24 in \citet{hu2021near},  for $\Delta_j, \varepsilon > 0$ and $r(j)\in \bbN$, 
\begin{align}
\label{eq:tail_sum}
	\sum_{r=r(j)+1}^{R}2^{r-1}\bbP\left[J_{r-1} = j\right] 
	&< \frac{c_1c_2}{\Delta_j \min\{\Delta_j, \varepsilon\}}\cdot \exp(-2^{r(j)}\Delta_j \min\{\Delta_j, \varepsilon\} / c_2).
\end{align}

Let $r(j) = \left\lceil \log_2\left( \frac{c_2(\log K)}{\Delta_j\varepsilon}\right) \right\rceil$. $\forall j$ such that $\Delta_j > \varepsilon$, $\sum_{r=1}^{R} 2^{r-1}\bbP[J_{r-1} = j]$ can be bounded as
$$
\sum_{r=1}^{r(j)} 2^{r-1}\bbP[J_{r-1} = j] +\sum_{r=r(j)+1}^{R} 2^{r-1}\bbP[J_{r-1} = j] 
< \left(\sum_{r=1}^{r(j)} 2^{r-1}\frac{1}{j}\right) + \frac{c_1c_2}{\Delta_j \varepsilon}\cdot \exp(-2^{r(j)}\Delta_j \varepsilon / c_2)
$$
$$
< 2^{r(j)}\frac{1}{j} + \frac{c_1c_2}{\Delta_j \varepsilon}\cdot \exp(-2^{r(j)}\Delta_j \varepsilon / c_2) \\
\leq \frac{2 c_2}{\Delta_j \varepsilon}\cdot \frac{\log K}{j} + \frac{c_1c_2}{\Delta_j \varepsilon}\cdot \frac{1}{K},
$$
where the first inequality holds by Lemma~\ref{lem:monotonicity} (\emph{since it is assumed that $B=1$ for the Algorithm~\ref{alg:main} in the theorem statement}) and Equation~\ref{eq:tail_sum}, the second inequality holds by $\sum_{r=1}^{r(j)} 2^{r-1} = 2^{r(j)}-1 < 2^{r(j)}$, and the third inequality holds by taking the value of $r(j)$.
Therefore,

$$
\text{Regret}_{\downarrow} < \sum_{j:\Delta_j > \varepsilon}^{K}\Delta_{j}\left(\frac{2 c_2}{\Delta_j \varepsilon} \frac{\log K}{j} + \frac{c_1c_2}{\Delta_j \varepsilon}\frac{1}{K}\right)
\leq  \frac{2 c_2}{\varepsilon}\sum_{j:\Delta_j > \varepsilon}^{K}\left(\frac{\log K}{j} + \frac{c_1}{K}\right) = O\left( \frac{\log^2 K}{\varepsilon} \right).
$$

By putting the analysis for $\text{Regret}_{\uparrow}$ and $\text{Regret}_{\downarrow}$ together, we have proved that the pseudoregret is bounded by $O\left(\frac{\log K}{\Delta_{\min}} + \frac{\log^2 K}{\varepsilon}\right)$.
\end{proof}

\section{Optimal Rate for a Weaker Deterministic Setting of the Open Problem}
\label{sec:main_det}
We notice that for both the existing result and our new result, the gap the the known lower bound appears with the DP factor $\varepsilon$ rather than the $\Delta_{\min}$.
This motivates us to study a weaker setting of the open problem to focus on differential privacy regardless of the sampling error in the observed losses.
Specifically, we study the same open problem but with the assumption that the distributions $\calP_j$ ($j\in [K]$) concentrate on the single value $\mu_j$, i.e. $\bbP_{\ell_j\sim \calP_j}[\ell_j=\mu_j] = 1$, and we call this weaker setting as the \textit{deterministic setting}.

Although the deterministic setting may seem limited in practical relevance, it provides a simplified environment to isolate and examine key challenges in the original problem.
Interestingly, a direct application of existing analyses and algorithms from the original setting still results in the same extra $\log K$ factor, suggesting that the underlying complexity persists.
Motivated by this, we undertake a deeper investigation of this weaker setting.


\subsection{Lower bound for the deterministic setting} 
Although \cite{hu2021near} has shown the lower bound $\Omega\left(\frac{\log K}{\varepsilon}\right)$ for the original open problem, their result does not directly apply to the deterministic setting, as the worst-case instance they construct falls outside this weaker setting.
Their construction relies on stochasticity and does not satisfy the assumptions of the deterministic case.
Instead, we develop a new lower bound instance that achieves the same rate of $\frac{\log K}{\varepsilon}$, and it applies to both the deterministic setting and the original problem.
The lower bound result is stated in the following theorem.
\begin{theorem}[Lower bound for the deterministic setting.]
\label{thm:lower_det}
	For any $\varepsilon$-differentially private online learning algorithm $\calM$, $K\in \bbN$ and $\Delta_{\min}$, $\exists (u_1, \cdots, u_K)\in [0, 1]^{K}$ s.t. at the deterministic setting,
	$$\mathrm{PseudoRegret}(\calM; T, \calP_1, \cdots, \calP_K) \geq c_1\frac{\log K}{\varepsilon},$$
	where $c_1$ is a universal constant independent of $K$,$\varepsilon$ and $(\mu_1, \cdots, \mu_k)$.
	Moreover, the sorted $(u_{(1)}, \cdots, u_{(K)})$ in the worst instance construction is $(0, \Delta_{\min}, \Delta_{\min}, 1, \cdots, 1)$
\end{theorem}
\begin{remark}
While the lower bound rate remains the same as in prior work, the proof is simpler and more self-contained, without needing to use an interesting DP version of Fanos' method~\citep{acharya2021differentially} that apparently overkill the problem.	
\end{remark}
\begin{remark}[Connection between Theorem \ref{thm:lower_det} and open problem]
A key mystery in prior work is that the worst-case instance for the non-private version of the problem is given by $(u_{(1)}, \cdots, u_{(K)})=(0, \Delta_{\min}, \cdots, \Delta_{\min})$, where all suboptimal actions share the same loss distributions.
However, this instance does not correspond to the hardest case (to come up with an analysis) for the private version of the problem. -- prior analysis in \cite{hu2021near} has proved the tightest rate $\Theta\left(\frac{\log K}{\Delta_{\min}} + \frac{\log K}{\varepsilon}\right)$ for this particular instance, but not able to prove this rate for all instances.
Our new lower bound construction sheds light on this discrepancy: the set of instances exhibiting regret of at least $\Omega(\frac{\log K}{\varepsilon})$ under privacy constraints is broader than the hardest instances in the non-private setting.
As a result, establishing the $O(\frac{\log K}{\varepsilon})$ term in the upper bound for all instances is inherently more difficult than analyzing just the worst-case instance in the non-private regime.
\end{remark}

\subsection{Upper bound for the deterministic setting}
\paragraph{Suboptimal results by extending the analysis for the original open problem.} First, we will see how the algorithm together with the similar analysis in Section~\ref{sec:main} works for the deterministic setting.
We can repeat the analysis in Theorem~\ref{thm:main} without considering the sampling errors, and get the following rate as a corollary; the detailed argument is in Appendix~\ref{sec:app_ext_det}.
\begin{corollary}[Extension from Theorem~\ref{thm:main}.]
\label{cor:ext_det}
	When $\calQ_{\varepsilon}$ is the Laplace distribution $\mathrm{Lap}(\frac{2}{\varepsilon})$, the Exponential distribution $\mathrm{Exp}(\frac{1}{\varepsilon})$, or the Gumbel distribution $\mathrm{Gumbel}(\frac{2}{\varepsilon})$, Algorithm~\ref{alg:main} is $\varepsilon$-differentially private and satisfies the gaurantee for the deterministic setting:
	\begin{align*}
		\mathrm{PseudoRegret}(\text{RNM-FTNL}(B, \calQ_{\varepsilon}); T, \calP_1, \cdots, \calP_K) = O\left(\frac{\log^2 K}{\varepsilon}\right)
	\end{align*}
\end{corollary}
Unfortunately, by comparing the rate with the lower bound, there is still an extra log factor in $K$, same as what it is for the original open problem.

\paragraph{The tight upper bound for the deterministic setting.} 
We first provide an algorithm in the same framework of variants of RNM-FTNL.
The algorithm is with a new specification of $B=0$ and $Q_{\varepsilon}$ as Exponential distribution or Gumbel distribution in Algorithm~\ref{alg:main}.
Notice that we are not sticking with the Laplace distribution that is used in the original RNM-FTNL \citep{hu2021near} for this setting.
This is because the report-noisy-max mechanism with Gumbel nosie is known as Exponential mechanism~\citep{qiao2021oneshot}, which has explicit forms for the probability of each action as an output.
The tractable expression of the regret is soft-max like and let us to derive our tight analysis.
In addition, we can make a similar conclusion for the Exponential distribution by a reduction since the previous study~\citep{mckenna2020permute} shows it is consistently better than the Gumbel distribution.
Nevertheless, we are not able to prove the same rate for the Laplace distribution, and whether it brings the same rate remains unknown.

With this new algorithm, we are able to prove the optimal rate for the deterministic setting, as stated in the following theorem.
\begin{theorem}[Main result 2: optimal rate for the deterministic setting.]
\label{thm:upper_det}
	When specifying $B=0$ and $\calQ_{\varepsilon}$ as the Exponential distribution $\mathrm{Exp}(\frac{1}{\varepsilon})$ or the Gumbel distribution $\mathrm{Gumbel}(\frac{2}{\varepsilon})$, Algorithm~\ref{alg:main} is $\varepsilon$-differentially private and satisfies the guarantee for the deterministic setting
	\begin{align*}
		\mathrm{PseudoRegret}(\text{RNM-FTNL}(B, \calQ_{\varepsilon}); T, \calP_1, \cdots, \calP_K) = O\left( \frac{\log K}{\varepsilon}\right).
	\end{align*}
	Moreover, this rate is optimal for the deterministic setting.
\end{theorem}
We show two lemmas for the soft-max like function $f(x) = \frac{\sum_{i=1}^K 2^xa_i e^{- 2^xa_i}}{\sum_{i=1}^K e^{-2^xa_i}}$; proofs are done by some calculas in Section~\ref{sec:app_lem_soft}.
\begin{lemma}
\label{lem:derivative}
	For any $i\in [K], a_i\in \bbR$, $f(x) = \frac{\sum_{i=1}^K 2^xa_i e^{- 2^xa_i}}{\sum_{i=1}^K e^{-2^xa_i}}$ has the property $f'(x)\leq \log 2\cdot f(x)$.
\end{lemma}
\begin{lemma}
\label{lem:calc}
	For any $0=a_1<a_2\leq, \cdots, \leq a_K$, $\sum_{r=1}^{\infty} \frac{\sum_{i=1}^{K}2^{r}a_i\exp\left( -2^{r}a_i \right)}{\sum_{i=1}^K\exp\left( -2^{r}a_i\right)} \leq O(\log K).$
\end{lemma}
\begin{proof}[Proof sketch for Theorem~\ref{thm:upper_det}]
	Report-noisy-max mechanism with Gumbel noise is equivalent to Exponential mechanism~\citep{gumbel1954statistical, qiao2021oneshot}, so we can have a tractable expression for $\bbP[J_r=j]$:
$\bbP\left[J_r = j|\forall i\in [K], G_{r, i} \right] = \frac{\exp\left( \varepsilon\cdot (-G_{r, j} ) \right)}{\sum_{i=1}^K\exp\left( \varepsilon\cdot (-G_{r, i}) \right)}.$
Moreover, at the deterministic setting, $\bbP[J_r=j]=\bbP\left[J_r = j|\forall i\in [K], G_{r, i} =2^{r-2}\mu_i\right]$. Therefore the regret has this tractable expression:
$$
O(1) + \sum_{r=3}^{R} \sum_{j=1}^{K}2^{r-1}\Delta_{j}\frac{\exp\left( -2^{r-2}\Delta_j\varepsilon \right)}{\sum_{i=1}^K\exp\left( -2^{r-2}\Delta_i\varepsilon\right)}.
$$
If we define this soft-max like function $f(x) := \frac{\sum_{i=1}^K 2^xa_i e^{- 2^xa_i}}{\sum_{i=1}^K e^{-2^xa_i}}$,
Lemma~\ref{lem:calc} proves $\forall 0=a_1<a_2\leq\cdots \leq a_k$, $\sum_{x=1}^{\infty}f(x)=O(\log K)$. 
We can specify $a_i=\Delta_i\varepsilon$ and finish proving that the above regret is $O(\frac{\log K}{\varepsilon})$.
The full proof with the detailed calculation and for the Exponential noise is in Appendix~\ref{sec:app_upp_det}.
\end{proof}
\begin{remark} We connect Theorem \ref{thm:upper_det} back to the original open problem by the points below.
\begin{enumerate}[leftmargin=*]
\item We provide a tight analysis for this deterministic setting, a special setting of the original open problem, while the analysis from the stochastic setting implies a suboptimal result.
\item Comparing the analyses across the two settings, we hypothesize that the current analysis for the original open problem may be loose due to its simplification -- specifically, it considers losses only in relation to each action and the optimal one (Equation~\ref{eq:tail_main}).
In contrast, the tighter analysis for the deterministic setting jointly considers all actions through a softmax-like function $f(x)$ when analyzing the regret.
This suggests that a more holistic treatment of the action set may be necessary to achieve tighter bounds in the original setting.
\item Motivated by the above discussion, we propose the following conjecture for the original open problem: for all $j$ s.t. $\varepsilon < \Delta_j$,
$
\bbP\left[J_r = j\right]\leq c_1\cdot \frac{\exp\left( -2^{r+1}\Delta_j\varepsilon/c_2 \right)}{1 + \sum_{j:\Delta_j>\varepsilon}\exp\left( -2^{r+1}\Delta_i\varepsilon/c_2\right)}.
$
This conjecture is strictly stronger than Equation~\ref{eq:tail_main}. If it holds, then the open problem could be resolved with a regret bound of $\Theta\left(\frac{\log K}{\Delta_{\min}}+\frac{\log K}{\varepsilon}\right)$, using an argument similar to our analysis for the deterministic setting.
At present, however, we are unable to prove or disprove the conjecture.

\end{enumerate}
\end{remark}

\section{Conclusion}
In this paper, we propose a new upper bound for the open problem that is independent of $T$ and depends only logarithmically on $K$.
In addition, we focus on a weaker variant of the problem, where the losses are assumed to be deterministic.
We present a new analysis for the deterministic setting and establish a tight bound, offering insights that may inform future progress on the original open problem.

%% file: appendix.tex
%

\section{Proof of the Property of Binomial Distribution}
\label{sec:app_bern}
\begin{lemma}
\label{lem:bern}
	Suppose $F(k; n, p)$ is the cumulative density function (CDF) of the binomial distribution $\calB(n, p)$. For any $0\leq p_1 < p_2\leq 1$, $F(k; n, p_1)\geq F(k; n, p_2)$.
\end{lemma}
\begin{proof}[Proof of Lemma~\ref{lem:bern}.]
	Suppose $F_{\rm beta-dist}(x; \alpha, \beta)$ is the CDF of beta-distribution.
	It has been proved the equivalence between the two CDFs~\citep{wadsworth1960introduction}:
	$$
	F(k; n, p) = F_{\rm beta-dist}(1-p; n-k, k+1).
	$$
	Therefore, for any $p_1 < p_2$, 
	$$F(k; n, p_1) = F_{\rm beta-dist}(1-p_1; n-k, k+1)\geq F_{\rm beta-dist}(1-p_2; n-k, k+1) = F(k; n, p_2)$$
\end{proof}

\section{Proof of Lemma~\ref{lem:monotonicity}}
\label{sec:app_monotonicity}

\begin{proof}[Proof of Lemma~\ref{lem:monotonicity}.]
	Let $N_{r, j} = -G_{r, j}+Q_{r, j}$ and denote $F_{A}(x)$ as the cumulative density function for any random variable $\bbP[A\leq x]$. 
	We can first prove for any $j_1 < j_2$ and $x\in \bbR$, $F_{N_{r, j_1}}(x) \leq F_{N_{r, j_2}}(x)$.
	To see its correctness, we can decompose $F_{N_{r, j_1}}(x)$ as 
	$$F_{N_{r, j_2}}(x)= \int_{-\infty}^{\infty}\bbP[-G_{r, j_1}\leq x-s]f_{Q_{r, j_1}}(s)ds= \int_{-\infty}^{\infty}(1 - F_{G_{r, j_1}}(s-x))f_{Q_{r, j_1}}(s)ds$$ 
	and similarly $F_{N_{r, j_2}}(x)=\int_{-\infty}^{\infty}(1 - F_{G_{r, j_2}}(s-x))f_{Q_{r, j_2}}(s)ds$.
	 
	Moreover, \emph{because $B=1$ is specified for the algorithm}, $G_{r, j}$ is from the binomial distribution $\calB(2^{r-1}, \mu_j)$.
	Binomial distribution has the property~(\citet{wadsworth1960introduction}; Appendix~\ref{sec:app_bern})
	$$\mu_{j_1}\leq\mu_{j_2}\Rightarrow F_{G_{r, j_1}}(x) \geq F_{G_{r, j_2}}(x).$$
	With this property, we can show $F_{N_{r, j_1}}(x) \leq F_{N_{r, j_2}}(x)$ by
	\begin{align*}
		\int_{-\infty}^{\infty}(1 - F_{G_{r, j_1}}(s-x))f_{Q_{r, j_1}}(s)ds
		\leq \int_{-\infty}^{\infty}(1 - F_{G_{r, j_2}}(s-x))f_{Q_{r, j_2}}(s)ds
	\end{align*}
	Now we turn to prove $\bbP[J_r = j_1]\geq \bbP[J_r = j_2]$ for $j_1 < j_2$. Let $H = \max_{j\neq j_1, j_2}N_{r, j}$ and let $N'_{r, j_2}$ be a random variable which is independent of $N_{r, j_2}$  but has the same distribution as $N_{r, j_2}$.
	By applying $F_{N_{r, j_1}}(x) \leq F_{N_{r, j_2}}(x)$ proved above, we have
	\begin{align*}
		\bbP[J_r = j_1] &= \bbP[N_{r, j_1}> \max\{N_{r, j_2}, H\}] = \int_{-\infty}^{\infty} (1 - F_{N_{r, j_1}}(s)) f_{\max\{N_{r, j_2}, H\}}(s)ds \\
		&\geq \int_{-\infty}^{\infty} (1 - F_{N_{r, j_2}'}(s)) f_{\max\{N_{r, j_2}, H\}}(s)ds\\
		&= \bbP[N_{r, j_2}'> \max\{N_{r, j_2}, H\}] = \bbP[N_{r, j_2}> \max\{N_{r, j_2}', H\}].
	\end{align*}
	Because $H$ and $N_{r, j_2}'$ are independent, by applying $F_{N_{r, j_1}}(x) \leq F_{N_{r, j_2}}(x) = F_{N_{r, j_2}'}(x)$ again,
	$
		F_{\max\{N_{r, j_2}', H\}}(x) = F_{N_{r, j_2}'}(x)\cdot F_{H}(x) \geq F_{N_{r, j_1}}(x)\cdot F_{H}(x) = F_{\max\{N_{r, j_1}, H\}}(x).
	$
	Therefore 
	\begin{align*}
		\bbP[J_r = j_1]&\geq \bbP[N_{r, j_2}> \max\{N_{r, j_2}', H\}]  = \int_{-\infty}^{\infty} F_{\max\{N_{r, j_2}', H\}}(s) f_{N_{r, j_2}}(s)ds \\
		&\geq \int_{-\infty}^{\infty} F_{\max\{N_{r, j_1}, H\}}(s) f_{N_{r, j_2}}(s)ds = \bbP[N_{r, j_2}> \max\{N_{r, j_1}, H\}] = \bbP[J_r = j_2].
	\end{align*}
	Finaly, we are going to show $\bbP[J_r = j]\leq \frac{1}{j}$. This can be derived by
	$
	1 = \sum_{i=1}^{K} \bbP[J_r = i] \geq \sum_{i=1}^{j} \bbP[J_r = i] \geq \sum_{i\leq j} \bbP[J_r = j] = j\cdot \bbP[J_r = j].
	$
\end{proof}

\section{Full Proof of Theorem~\ref{thm:main}}
\label{sec:app_proof_main}
We first prove a lemma for the other two noise distributions $\calQ_{\varepsilon}$
\begin{lemma}
\label{lem:more_dist}
	$\calQ_{\varepsilon}$ is $\mathrm{Exp}(\frac{1}{\varepsilon})$ or $\mathrm{Gumbel}(\frac{2}{\varepsilon})$, there exists universal constants $c_1, c_2>0$ such that
	$$
	\bbP\left[J_r = j\right]\leq c_1\cdot \exp(-2^r\Delta_j \min\{\Delta_j, \varepsilon\} / c_2).
$$
\end{lemma}

\begin{proof}[Proof of Lemma~\ref{lem:more_dist}.]
The proof for $\calQ_{\varepsilon}=\mathrm{Exp}(\frac{1}{\varepsilon})$ is almost the same as their proof for $\calQ_{\varepsilon}=\mathrm{Lap}(\frac{2}{\varepsilon})$:
\begin{align*}
	\bbP\left[J_r = j\right] &\leq \bbP\left[-G_{r, j} + Q_{r, j} > -G_{r, 1} + Q_{r, 1}\right]\\
	&\leq  \bbP\left[ G_{r, j} -G_{r, 1} \leq 2^{r-1}\frac{\Delta_j}{2}\right] + \bbP\left[ Q_{r, j} -Q_{r, 1} \geq 2^{r-1}\frac{\Delta_j}{2}\right].
\end{align*}
From the Hoeffding inequality, 
$$\bbP\left[ G_{r, j} -G_{r, 1} \leq 2^{r-1}\frac{\Delta_j}{2}\right] = \bbP\left[ G_{r, j} -G_{r, 1} - 2^{r-1}\Delta_j \leq -2^{r-1}\frac{\Delta_j}{2}\right] \leq \exp\left(-2^{r-1}\frac{\Delta_j^2}{4}\right).$$
By the cdf of any eponential distribution, 
$$\bbP\left[ Q_{r, j} -Q_{r, 1} \geq 2^{r-1}\frac{\Delta_j}{2}\right] \leq \bbP\left[ Q_{r, j} \geq 2^{r-1}\frac{\Delta_j}{2}\right] \leq \exp\left(-\varepsilon 2^{r-1}\frac{\Delta_j}{2}\right).$$
Therefore, for $\calQ_{\varepsilon}=\mathrm{Exp}(\frac{1}{\varepsilon})$, $\bbP\left[J_r = j\right]\leq 2\cdot \exp(-2^r\Delta_j \min\{\Delta_j, \varepsilon\}/8)$.
	
As for $\calQ_{\varepsilon}=\mathrm{Gumbel}(\frac{2}{\varepsilon})$, it is known that the report-noisy-max with gumbel noise is equivalent to Exponential Mechanism~\citep{mcsherry2007mechanism,qiao2021oneshot}, which is
$$\bbP\left[J_r = j|\forall i\in [K], G_{r, i} \right] = \frac{\exp\left( \varepsilon\cdot (-G_{r, j} ) \right)}{\sum_{i=1}^K\exp\left( \varepsilon\cdot (-G_{r, i}) \right)}.$$
We bound $\bbP\left[J_r = j\right]$ as
\begin{align*}
	\bbP\left[J_r = j\right] & =\bbE_{\forall i\in [K], G_{r, i}}[\bbP\left[J_r = j|\forall i\in [K], G_{r, i} \right]]\\
	&= \bbE_{\forall i\in [K], G_{r, i}}\left[\frac{\exp\left( -\varepsilon\cdot (-G_{r, j} ) \right)}{\sum_{i=1}^K\exp\left( -\varepsilon\cdot (-G_{r, i}) \right)}\right]\\
	&\leq \bbE_{\forall i\in [K], G_{r, i}}\left[\frac{\exp\left( \varepsilon\cdot (-G_{r, j}) \right)}{\exp\left( \varepsilon\cdot (-G_{r, 1} ) \right) + \exp\left( \varepsilon\cdot (-G_{r, j}) \right)}\right]\\
	& = \bbE\left[\frac{1}{\exp\left( \varepsilon\cdot (G_{r, j}-G_{r, 1}) \right) + 1}\right].
\end{align*}
Denote the event $\mathcal{E}$ as $G_{r, j}-G_{r, 1} \geq \frac{1}{2}2^{r-1}\Delta_j$, because $\frac{1}{\exp\left( \varepsilon\cdot (G_{r, j}-G_{r, 1}) \right) + 1}\leq 1$ is always true,
$$
\bbP\left[J_r = j\right] \leq \bbE\left[\frac{1}{\exp\left( \varepsilon\cdot (G_{r, j}-G_{r, 1}) \right) + 1}| \mathcal{E} \right] + (1 - \bbP[\mathcal{E}])
$$
$$
\leq \frac{1}{\exp\left(  \frac{1}{2}2^{r-1}\Delta_j \varepsilon) \right) + 1} + (1 - \bbP[\mathcal{E}]) \leq \exp\left( -2^{r-1}\Delta_j \varepsilon/2)\right)  + (1 - \bbP[\mathcal{E}]).
$$
The bound for $1 - \bbP[\mathcal{E}] = \bbP[(G_{r, j}-G_{r, 1}) < 2^{r-1}\Delta_j/2]$ is
\begin{align*}
&\bbP\left[G_{r, j}-G_{r, 1}\right] \leq \exp\left(-2^{r-1}\Delta_j^2/4\right)
\end{align*}
where the inequality is held by the Hoeffding inequality. Therefore, 
$$
\bbP\left[J_r = j\right] \leq \exp\left( -2^{r-1}\Delta_j \varepsilon/2)\right) + \exp\left(-2^{r-1}\Delta_j^2/4\right).
$$
Our proof for the case $\calQ_{\varepsilon}=\mathrm{Gumbel}(\frac{2}{\varepsilon})$ is complete.
\end{proof}

Now we show the full proof for Theorem~\ref{thm:main}.
\begin{proof}[Proof of Theorem~\ref{thm:main}.]
If we can prove Equation~\ref{eq:regret_main} for any $T:=2^{R}-1$ where $R$ is any non-negative integer, Equation~\ref{eq:regret_main} would also hold for arbitrary $T$, because Algorithm~\ref{alg:main} is independent of the $T$ and the regret of Algorithm~\ref{alg:main} is non-decreasing in $T$.
Therefore, we can assume $T:=2^{R+1}-1$ for some non-negative integer $R$ and can rewrite the pseudoregret (defined in Eqeation~\ref{eq:regret}) according to the Algorithm~\ref{alg:main}:
$$
\sum_{t=1}^T \bbE\left[\Delta_{I_t}\right] = \sum_{r=1}^{R} \sum_{t=2^{r-1}}^{2^{r}-1}\bbE\left[\Delta_{I_t}\right] = \sum_{r=1}^{R} 2^{r-1}\bbE\left[\Delta_{J_{r-1}}\right] = \sum_{r=1}^{R} 2^{r-1}\sum_{j=1}^{K}\Delta_{j}\bbP[J_{r-1} = j]
$$
$$ 
= \sum_{j=1}^{K}\Delta_{j}\sum_{r=1}^{R} 2^{r-1}\bbP[J_{r-1} = j] = \underbrace{\sum_{j:\Delta_j\leq \varepsilon}\Delta_{j}\sum_{r=1}^{R} 2^{r-1}\bbP[J_{r-1} = j]}_{\text{Regret}_{\uparrow}} + \underbrace{\sum_{j:\Delta_j > \varepsilon}^{K}\Delta_{j}\sum_{r=1}^{R} 2^{r-1}\bbP[J_{r-1} = j]}_{\text{Regret}_{\downarrow}}
$$

According to the Lemma 9 in \citet{hu2021near} and  Lemma~\ref{lem:more_dist}, for all three noise distributions $\calQ_{\varepsilon}$, there exists universal constants $c_1, c_2>0$ such that 
	\begin{align}
	\label{eq:tail}
	\bbP\left[J_r = j\right]\leq c_1\cdot \exp(-2^{r+1}\Delta_j \min\{\Delta_j, \varepsilon\} / c_2),
	\end{align}
and similar to the proof for theorem 24 in \citet{hu2021near}, for $\Delta_j, \varepsilon > 0$, we can calculate
\begin{align}
\label{eq:tail_sum_2}
	\sum_{r=r(j)+1}^{R}2^{r-1}\bbP\left[J_{r-1} = j\right] 
	&\leq \sum_{r=r(j)+1}^{R}2^{r-1} c_1\cdot \exp(-2^r\Delta_j \min\{\Delta_j, \varepsilon\} / c_2)\nonumber\\
	&< c_1 \sum_{r=r(j)+1}^{R}\sum_{t=2^{r-1}+1}^{2^{r}} \exp(-t\Delta_j \min\{\Delta_j, \varepsilon\} / c_2)\nonumber\\
	&< c_1\sum_{t=2^{r(j)}+1}^{\infty} \cdot \exp(-t\Delta_j \min\{\Delta_j, \varepsilon\} / c_2)\nonumber\\
	&< c_1\int_{2^{r(j)}}^{\infty} \cdot \exp(-t\Delta_j \min\{\Delta_j, \varepsilon\} / c_2)dt\nonumber\\
	&= \frac{c_1c_2}{\Delta_j \min\{\Delta_j, \varepsilon\}}\cdot \exp(-2^{r(j)}\Delta_j \min\{\Delta_j, \varepsilon\} / c_2)
\end{align}

We first bound $\text{Regret}_{\downarrow}$. Let $r(j) = \left\lceil \log_2\left( \frac{c_2(\ln K)}{\Delta_j\varepsilon}\right) \right\rceil$. Then for any $j$ s.t. $\Delta_j > \varepsilon$,
\begin{align*}
\sum_{r=1}^{R} 2^{r-1}\bbP[J_{r-1} = j]
&=\sum_{r=1}^{r(j)} 2^{r-1}\bbP[J_{r-1} = j] +\sum_{r=r(j)+1}^{R} 2^{r-1}\bbP[J_{r-1} = j]\\
&< \left(\sum_{r=1}^{r(j)} 2^{r-1}\frac{1}{j}\right) + \frac{c_1c_2}{\Delta_j \varepsilon}\cdot \exp(-2^{r(j)}\Delta_j \varepsilon / c_2)\\
&< 2^{r(j)}\frac{1}{j} + \frac{c_1c_2}{\Delta_j \varepsilon}\cdot \exp(-2^{r(j)}\Delta_j \varepsilon / c_2)\\
&\leq \frac{2 c_2}{\Delta_j \varepsilon}\cdot \frac{\ln K}{j} + \frac{c_1c_2}{\Delta_j \varepsilon}\cdot \frac{1}{K},
\end{align*}
where the first inequality holds by Lemma~\ref{lem:monotonicity} (\emph{since it is assumed that $B=1$ for the Algorithm~\ref{alg:main} in the theorem statement}) and Equation~\ref{eq:tail_sum_2}, the second inequality holds by $\sum_{r=1}^{r(j)} 2^{r-1} = 2^{r(j)}-1 < 2^{r(j)}$, and the third inequality holds by taking the value of $r(j)$.
Therefore,
$$
\text{Regret}_{\downarrow} = \sum_{j:\Delta_j > \varepsilon}^{K}\Delta_{j}\sum_{r=1}^{R} 2^{r-1}\bbP[J_{r-1} = j] < \sum_{j:\Delta_j > \varepsilon}^{K}\Delta_{j}\left(\frac{2 c_2}{\Delta_j \varepsilon}\cdot \frac{\ln K}{j} + \frac{c_1c_2}{\Delta_j \varepsilon}\cdot \frac{1}{K}\right)
$$
$$
\leq  \frac{2 c_2}{\varepsilon}\cdot \sum_{j:\Delta_j > \varepsilon}^{K}\frac{\ln K}{j} + \frac{c_1c_2}{\varepsilon}\cdot \sum_{j:\Delta_j > \varepsilon}^{K}\frac{1}{K} = O\left( \frac{(\ln K)^2}{\varepsilon} \right).
$$
	
The remaining is to bound $\text{Regret}_{\uparrow}$, which is the same as a part of proof for Theorem 9 in ~\citet{hu2021near}.
For completeness, we illustrate the details here.
The idea is to group $j$. Define $\Delta_{(l)}:= 2^{l-1}\Delta_{\min}$ and denote $H_l := \{j: \Delta_{(l)} \leq \Delta_j < \Delta_{(l+1)}\}\cap \{j: \Delta_j < \varepsilon, j\geq 2\}$.
Then for any $j\in H_l$, we pick $r(j) := \tau_l = \left\lceil \frac{c_2\ln(|H_l|)}{\Delta_{(l)}^2} \right\rceil$.
\begin{align*}
	\text{Regret}_{\uparrow} &= \sum_{j:\Delta_j\leq \varepsilon}\Delta_{j}\cdot \left(\sum_{r=1}^{r(j)} 2^{r-1}\bbP[J_{r-1} = j] +\sum_{r=r(j)+1}^{R} 2^{r-1}\bbP[J_{r-1} = j]\right) \\
	& = \sum_{l=1}^{\infty}\sum_{j\in H_l} \Delta_{j}\cdot \left(\sum_{r=1}^{\tau_l} 2^{r-1}\bbP[J_{r-1} = j] +\sum_{r=\tau_l+1}^{R} 2^{r-1}\bbP[J_{r-1} = j]\right) \\
	& = \sum_{l=1}^{\infty} \left(\sum_{r=1}^{\tau_l}2^{r-1} \sum_{j\in H_l}\Delta_j\cdot \bbP[J_{r-1} = j] \right) + \sum_{l=1}^{\infty} \left(\sum_{j\in H_l}\Delta_j\cdot  \sum_{r=\tau_l+1}^{R}2^{r-1} \bbP[J_{r-1} = j] \right)\\
	& \leq \sum_{l=1}^{\infty} \left(\sum_{r=1}^{\tau_l}2^{r-1}\right)\cdot 2\Delta_{(l)}+ \sum_{l=1}^{\infty} \left(\sum_{j\in H_l}\Delta_j\cdot \frac{c_1c_2}{\Delta_j ^2}\cdot \exp(-2^{r(j)}\Delta_j^2 / c_2) \right)\\
	& < \sum_{l=1}^{\infty} 2^{\tau_l+2}\Delta_{(l)}+ \sum_{l=1}^{\infty} \left(|H_l|\cdot\frac{c_1c_2}{\Delta_{(l)}}\cdot \exp(-2^{r(j)}\Delta_{(l)}^2 / c_2) \right)\\
	& \leq \sum_{l=1}^{\infty}\frac{8c_2\ln(|H_l|)}{\Delta_{(l)}} + \sum_{l=1}^{\infty} \frac{c_1c_2}{\Delta_{(l)}} \\
	& \leq \frac{8c_2\ln K + c_1c_2}{\Delta_{\min}}\sum_{l=1}^{\infty}\frac{1}{2^{l-1}}\\
	& = \frac{8c_2\ln K + c_1c_2}{\Delta_{\min}},
	\end{align*}
The first inequality is because Equation~\ref{eq:tail_sum} and the fact that for $j\in H_l$, $\sum_{j\in H_l}\Delta_j\cdot \bbP[J_{r-1} = j]\leq 2\Delta_{(l)}\sum_{j\in H_l}\bbP[J_{r-1} = j]\leq 2\Delta_{(l)}$; the second inequality holds by $\sum_{r=1}^{\tau_l}2^{r-1} < 2^{\tau_l}$ and the fact that for $j\in H_l$, $\Delta_j\geq \Delta_{(l)}$; the third inequality holds by taking the value of $\tau_l$; the fourth inequality holds by the definition of $\Delta_{(l)}$ and the fact $|H_l|\leq K$.

Putting the analysis for $\text{Regret}_{\uparrow}$ and $\text{Regret}_{\downarrow}$ together, we have proved that the pseudoregret is bounded by $O\left(\frac{\log(K)}{\Delta_{\min}} + \frac{(\log K)^2}{\varepsilon}\right)$.
\end{proof}

\section{Proof of Corollary~\ref{cor:inst_ind}}
\label{sec:app_proof_inst_ind}
The proof follows the exact same steps as the proof for Corollary 11 in \citet{hu2021near}, which is also well-known as early as~\citet{audibert2009minimax}.
For completeness, we repeat the exact steps here.
\begin{proof}[Proof of Corollary~\ref{cor:inst_ind}]
	Let $\Delta^*:=\sqrt{\log K/T}$ be the critical gap. Then, for all actions $j$ that $\Delta_j<\Delta^*$, the can contribute the regret at most $T\cdot \Delta^*=\sqrt{T\log K}$.
	 To bound the contributions for actions $j$ that $\Delta_j\geq \Delta^*$, we can simply adapt the proof of our Theorem~\ref{thm:main} and Corollary~\ref{cor:more_dist} for only these actions, and the effective $\Delta_{\min}$ becomes $\Delta^*$.Therefore, the  bound for the overall regret becomes 
	 $$
	 O\left(\sqrt{T\log K} + \frac{\log K}{\Delta^*} + \frac{(\log K)^2}{\varepsilon}\right) = O\left(\sqrt{T\log K} + \frac{(\log K)^2}{\varepsilon}\right)
	 $$
\end{proof}

\section{Proof of Theorem~\ref{thm:lower_det}}
\label{sec:app_lower_bound}
The lower bound for the original setting, that is $\Omega\left(\frac{\log(K)}{\Delta_{\min}} + \frac{\log(K)}{\varepsilon}\right)$, is an application of Corollary 4 in \citet{acharya2021differentially}.
However, Corollary 4 in \citet{acharya2021differentially} requires a bounded KL divergence, while at our deterministic setting where each $\calP_i$ has probability $0$ on all values except $\mu_i$, the KL divergence between $\calP=\calP_1\times\cdots\times\calP_k$ and $\calP'$ is infinity when $\calP\neq \calP'$.
Therefore, we show an easy and standard construction for our setting.

\begin{proof}[Proof of Theorem~\ref{thm:lower_det}.]
	For any $l\in [K]$, define $\calP^{(l)} := \calP^{(l)}_1\times\cdots\times\calP^{(l)}_K$, where $\bbP_{\ell_i\sim\calP_i^{(l)}}[\ell_i=\mu_i^{(l)}]=1$, $\mu_l^{(l)} = 0$, $\mu_{l-1}^{(l)} = \mu_{l+1}^{(l)} = \Delta_{\min}$ and $\mu_i^{(l)}=1$ for all $i\notin [l-1, l+1]$.
	Here the subscript index is cyclic in $K$: $\mu_{0}^{(l)} = \mu_{K}^{(l)}$ and $\mu_{K+1}^{(l)} = \mu_{1}^{(l)}$.
	Suppose $\calA$ is any online algorithm that is $\varepsilon$-differentially private. When $K$ actions have the loss from $\calP^{(l)}$, denote $I_t^{(l)}$ is the action from $\calA$ and further for any length of the online procedure $T$, let $R^{(l)}(T)$ is the pseudoregret.
	Therefore 
	$$
	R^{(l)}(T) \geq \sum_{t=1}^T\bbP[I_t^{(l)}\notin [l-1, l+1]]
	$$
	One the other hand, because $\calA$ is differentially private, for any $l,l'\in [K]$, any action $i$, and any $t\geq T$,
	$$
	\bbP[I_t^{(l)} = i]\leq e^{t\cdot \varepsilon}\cdot \bbP[I_t^{(l')} = i].
	$$
	Therefore, 
	$$
	\bbP[I_t^{(l)}\notin [l-1, l+1]] = 1 - \sum_{\hat{l}=l-1}^{l+1}\bbP[I_t^{(l)}=\hat{l}] \geq 1 - \sum_{\hat{l}=l-1}^{l+1}\frac{e^{t\cdot \varepsilon}}{K-5}\sum_{l'\notin [l-2, l+2]}\bbP[I_t^{(l')}= \hat{l}].
	$$
	We take a sum of all $l\in [K]$:
	\begin{align*}
		\sum_{l=1}^K\bbP[I_t^{(l)}\notin [l-1, l+1]] &\geq K - \frac{e^{t\cdot \varepsilon}}{K-5}\sum_{l=1}^K\sum_{\hat{l}=l-1}^{l+1}\sum_{l'\notin [l-2, l+2]}\bbP[I_t^{(l')}= \hat{l}] \\
		 &= K - \frac{e^{t\cdot \varepsilon}}{K-5}\sum_{l'=1}^K\sum_{l\notin [l'-2, l'+2]}\sum_{\hat{l}=l-1}^{l+1}\bbP[I_t^{(l')}= \hat{l}]\\
		&\geq K - \frac{e^{t\cdot \varepsilon}}{K-5}\sum_{l'=1}^K3\cdot \bbP[I_t^{(l')}\notin [l'-1, l'+1]] \\
		&= K - \frac{3e^{t\cdot \varepsilon}}{K-5}\sum_{l=1}^K\bbP[I_t^{(l)}\notin [l-1, l+1]] 
	\end{align*}
	where the first equality holds by swiping the order of summations of $l$ and $l'$.
	This gives us $\sum_{l=1}^K\bbP[I_t^{(l)}\notin [l-1, l+1]] \geq \frac{K (K-5)}{3e^{t\cdot \varepsilon} + K-5}$.
	Thus, 
	$$
	\frac{1}{K}\sum_{l=1}^KR^{(l)}(T)\geq \sum_{t=1}^T\frac{K-5}{3e^{t\cdot \varepsilon} + K-5}\geq \sum_{t=1}^T\int_{t}^{t+1}\frac{K-5}{3e^{\tau\cdot \varepsilon} + K-5}d\tau = \int_{1}^{T+1}\frac{K-5}{3e^{t\cdot \varepsilon} + K-5}dt
	$$
	where the second inequality holds because $\frac{K-5}{3e^{t\cdot \varepsilon} + K-5}$ is monotonically decreasing.
	The antiderivatives for $g(t)=\frac{K-5}{3e^{t\cdot \varepsilon} + K-5}$ are $\frac{\ln\left(\frac{e^{t\varepsilon}}{3e^{t\varepsilon} + K - 5}\right)}{\varepsilon} + C$ for any constant $C$, which implies:
	$$
	\frac{1}{K}\sum_{l=1}^KR^{(l)}(T)\geq \int_{1}^{T+1}\frac{K-5}{3e^{t\cdot \varepsilon} + K-5}dt = \frac{\ln\left(\frac{e^{(T+1)\varepsilon}}{3e^{(T+1)\varepsilon} + K - 5}\cdot \frac{3e^{\varepsilon} + K - 5}{e^{\varepsilon}}\right)}{\varepsilon}.
	$$
	From here, it implies that there exists $l_T^*$ s.t. 
	$$
	R^{(l_T^*)}(T)\geq \frac{\ln\left(\frac{e^{(T+1)\varepsilon}}{3e^{(T+1)\varepsilon} + K - 5}\cdot \frac{3e^{\varepsilon} + K - 5}{e^{\varepsilon}}\right)}{\varepsilon}.
	$$
	When $T\to\infty$,
	$$
	\lim_{T\to\infty}R^{(l_T^*)}(T)\geq \frac{\ln\left(\frac{e^{\varepsilon} + (K - 5)/3}{e^{\varepsilon}}\right)}{\varepsilon} = \frac{\ln\left(e^{\varepsilon} + (K - 5)/3\right) }{\varepsilon} - 1 \geq \frac{\ln ((K - 5)/3)  }{\varepsilon} - 1 = \Omega\left(\frac{\ln K  }{\varepsilon}\right).
	$$
\end{proof}

\section{Proof of Corollary~\ref{cor:ext_det}}
\label{sec:app_ext_det}
The proof for the deterministic setting is a straightforward extension from the proof for Theorem~\ref{thm:main} (the result at the original setting).
\begin{proof}[Proof sketch of Corollary~\ref{cor:ext_det}.]
	With the additional assumption at the deterministic setting that $\bbP_{\ell_j\sim\calP_j}[\ell_j=\mu_j]=1$, $\bbP[J_r=j]$ can be bounded in the form when $\calQ_\varepsilon$ is laplace distribution, exponential distribution, or gumbel distribution:
	\begin{align}
	\label{eq:tail_det}
	\bbP\left[J_r = j\right]\leq c_1\cdot \exp(-2^r\Delta_j \varepsilon / c_2)
	\end{align}
	for some universal constants $c_1, c_2>0$, a slight improvement from the bound $\bbP\left[J_r = j\right]\leq c_1\cdot \exp(-2^r\Delta_j \min\{\Delta_j, \varepsilon\} / c_2)$ (Equation~\ref{eq:tail}) at the original setting.
	Then by extending the similar derivation in the proof of Theorem~\ref{thm:main} (Section~\ref{sec:app_proof_main}), we can prove that the pseudo regret is bounded by $O\left(\frac{(\log K)^2}{\varepsilon}\right)$
\end{proof}

\section{Proof of the Properties for the Soft-Max Like Function}
\label{sec:app_lem_soft}
\begin{proof}[Proof of Lemma~\ref{lem:derivative}.]
This can be proved by calculating the derivatives $f'(x)$:
\begin{align*}
		&f'(x) = \frac{\left(\sum_{i=1}^K 2^xa_i e^{- 2^xa_i}\right)'}{\sum_{i=1}^K e^{-2^xa_i}} - \frac{\left(\sum_{i=1}^K 2^xa_i e^{- 2^xa_i}\right)\cdot \left(\sum_{i=1}^K e^{-2^xa_i}\right)'}{\left(\sum_{i=1}^K e^{-2^xa_i}\right)^2}\\
		& = \left(\log 2\cdot \frac{\sum_{i=1}^K 2^xa_i\cdot e^{-2^xa_i } }{\sum_{i=1}^K e^{-2^xa_i }} -\log2 \cdot\frac{\sum_{i=1}^K \left(2^xa_i\right)^2\cdot e^{-2^xa_i }  }{\sum_{i=1}^K e^{-2^xa_i }}\right) + \log 2 \cdot \frac{ (\sum_{i=1}^K 2^xa_i e^{-2^xa_i})^2}{(\sum_{i=1}^K e^{-2^xa_i})^2}\\
		& = (\log 2) f(x) - (\log 2)\frac{ (\sum_{i=1}^K (2^xa_i)^2e^{-2^xa_i})(\sum_{i=1}^K e^{-2^xa_i}) -(\sum_{i=1}^K 2^xa_i e^{-2^xa_i})^2}{(\sum_{i=1}^K e^{-2^xa_i})^2}\\
		&\leq (\log 2) f(x),
	\end{align*}
        where the last inequality is held by Cauchy Schwarz Inequality.	
\end{proof}

\begin{proof}[Proof of Lemma~\ref{lem:calc}.]
Let $f(x) = \frac{\sum_{i=1}^K 2^xa_i e^{- 2^xa_i}}{\sum_{i=1}^K e^{-2^xa_i}}$. Then, 
	\begin{align*}
	\sum_{r=1}^{\infty} \frac{\sum_{i=1}^{K}2^{r}a_i\exp\left( -2^{r}a_i \right)}{\sum_{i=1}^K\exp\left( -2^{r}a_i\right)} = \sum_{r=1}^{\infty} f(r) = \sum_{r=1}^{\infty} \left[\left(f(r) - \int_{r-1}^{r}f(x)dx\right) + \int_{r-1}^{r}f(x)dx \right].
	\end{align*}
	From the Lagrange's mean value theorem, $\int_{r-1}^{r}f(x)dx = f(x_r)$ for some $x_r\in [r-1, r]$.
	Therefore 
	\begin{align}
	\label{eq:dis_to_cont_one}
	f(r) - \int_{r-1}^{r}f(x)dx = f(r) - f(x_r)= \int_{x_r}^r f'(x) dx \leq \int_{x_r}^r \log2f(x) dx  \leq \log2\int_{r-1}^r f(x) dx,
	\end{align}
	where the first inequality holds by $f'(x)\leq \log 2\cdot f(x)$ that we just proved and the second inequality is true because $f(x)\geq 0$ for all $x$.
	With the Equation~\ref{eq:dis_to_cont_one}, we now have
	\begin{align}
	\label{eq:dis_to_cont}
	\sum_{r=1}^{\infty} \frac{\sum_{i=1}^{K}2^{r}a_i\exp\left( -2^{r}a_i \right)}{\sum_{i=1}^K\exp\left( -2^{r}a_i\right)}
	& \leq(\log2 + 1)\sum_{r=1}^{\infty}  \int_{r-1}^{r}f(x)dx = (\log2 + 1)\int_{0}^{\infty}f(x)dx.
	\end{align}
	
	The last thing is to bound $\int_{0}^{\infty}f(x)dx$. Notice that the antiderivatives for $ f(x) = \frac{\sum_{i=1}^K 2^xa_i e^{- 2^xa_i}}{\sum_{i=1}^K e^{-2^xa_i}}$ is $F(x) = -\frac{1}{
	\log 2}\log\left( \sum_{i=1}^K e^{- 2^xa_i} \right) + C$ for any constant $C$.
	Moreover, because $0 {\color{blue}=} a_1 < a_2\leq, \cdots, \leq a_K$,
	$$F(0) = -\frac{1}{
	\log 2}\log\left( \sum_{i=1}^K e^{- a_i} \right) + C \geq -\frac{1}{
	\log 2}\log\left( K \right) + C;\lim_{x\infty}F(x) =-\frac{1}{
	\log 2}\log\left( 1 \right) + C = C.$$ 	
	Therefore $\int_{0}^{\infty}f(x)dx = \lim_{x\to+\infty}F(x) - F(0) = \frac{2}{
	\log 2}\log(K)$. 	Taking this equality to Equation~\ref{eq:dis_to_cont}, our proof is complete.
\end{proof}

\section{Proof of Theorem~\ref{thm:upper_det}}
\label{sec:app_upp_det}
\begin{proof}[Proof of Theorem~\ref{thm:upper_det}.]
	We first prove for the gumbel distribution $\mathrm{Gumbel}(\frac{2}{\varepsilon})$. 
	It is known that the report-noisy-max with gumbel noise is equivalent to Exponential Mechanism~\citep{mcsherry2007mechanism,qiao2021oneshot}, which is
$$\bbP\left[J_r = j|\forall i\in [K], G_{r, i} \right] = \frac{\exp\left( \varepsilon\cdot (-G_{r, j} ) \right)}{\sum_{i=1}^K\exp\left( \varepsilon\cdot (-G_{r, i}) \right)}.$$
Because we are considering the deterministic setting, $G_{r, i}=2^{r-1}\mu_i$ with probability $1$.
Therefore,
$$\bbP\left[J_r = j \right] = \frac{\exp\left( -2^{r-1}\mu_i\varepsilon \right)}{\sum_{i=1}^K\exp\left( -2^{r-1}\mu_j\varepsilon\right)} = \frac{\exp\left( -2^{r-1}\Delta_i\varepsilon \right)}{\sum_{i=1}^K\exp\left( -2^{r-1}\Delta_j\varepsilon\right)}.$$
Then let $a_i = \Delta_i\varepsilon$ in Lemma~\ref{lem:calc} and we can show the upper bound for pseudoregret:
\begin{align*}
\sum_{t=1}^T \bbE\left[\Delta_{I_t}\right] &\leq 3 + \sum_{r=3}^{\infty} 2^{r-1}\sum_{j=1}^{K}\Delta_{j}\bbP[J_{r-1} = j] \leq 3 + 2\cdot \sum_{r=1}^{\infty} \frac{\sum_{i=1}^{K}2^{r}\Delta_{i}\exp\left( -2^{r}\Delta_i\varepsilon \right)}{\sum_{i=1}^K\exp\left( -2^{r}\Delta_j\varepsilon\right)} \\
& =3 + \frac{2}{\varepsilon}\cdot \sum_{r=1}^{\infty} \frac{\sum_{i=1}^{K}2^{r}\Delta_{i}\varepsilon\exp\left( -2^{r}\Delta_i\varepsilon \right)}{\sum_{i=1}^K\exp\left( -2^{r}\Delta_j\varepsilon\right)} \leq O\left(\frac{\log K}{\varepsilon}\right)
\end{align*}

We have proved the upper bound for $\calQ_{\varepsilon} = \mathrm{Gumbel}(\frac{2}{\varepsilon})$ and now we can prove the upper bound for the exponential distribution $\calQ_{\varepsilon} = \mathrm{Exp}(\frac{1}{\varepsilon})$.
To distinguish, $I_t$ is still the action from $\calQ_{\varepsilon} = \mathrm{Gumbel}(\frac{2}{\varepsilon})$, and we denote $I_t^{\rm exp}$ as the action from $\calQ_{\varepsilon} = \mathrm{Exp}(\frac{1}{\varepsilon})$ and $J_r^{\rm exp}$ as the output from report-noisy-max with the exponential noise.
\citet{mckenna2020permute} has proved (in their Theorem 2) that the report-noisy-max with exponential noise is consistently better than the exponential mechanism, which is equivalent to the report-noisy-max with gumbel noise~\citep{gumbel1954statistical, qiao2021oneshot}.
$$
\sum_{j=1}^{K}(-2^{r-1}\mu_j)\cdot \bbP[J_{r}^{\rm exp} = j] \geq \sum_{j=1}^{K}(-2^{r-1}\mu_j)\cdot \bbP[J_{r} = j]
\Rightarrow \sum_{j=1}^{K}\mu_j\cdot \bbP[J_{r}^{\rm exp} = j] \leq \sum_{j=1}^{K}\mu_j\cdot \bbP[J_{r} = j].
$$
Subtract $\mu_1$ from both sides, we have 
$
\sum_{j=1}^{K}\Delta_j\cdot \bbP[J_{r}^{\rm exp} = j]\leq \sum_{j=1}^{K}\Delta_j\cdot \bbP[J_{r} = j],
$
and then the pseudoregret when $\calQ_{\varepsilon} = \mathrm{Exp}(\frac{1}{\varepsilon})$ can be bounded by
\begin{align*}
\sum_{t=1}^T \bbE\left[\Delta_{I_t^{\rm exp}}\right]  \leq 3 + \sum_{r=1}^{R-2} 2^{r+1}\sum_{j=1}^{K}\Delta_{j}\bbP[J_{r+1}^{\rm exp} = j]  \leq 3 + \sum_{r=1}^{R-2} 2^{r+1}\sum_{j=1}^{K}\Delta_{j}\bbP[J_{r+1} = j],
\end{align*}
which now is the case of $\calQ_{\varepsilon} = \mathrm{Gumbel}(\frac{2}{\varepsilon})$ and bounded by $O\left(\frac{\log K}{\varepsilon}\right)$. 

We have proved the pseudoregret can be bounded by $O\left(\frac{\log K}{\varepsilon}\right)$ when specifying $B=0$ and $\calQ_{\varepsilon} = \mathrm{Gumbel}(\frac{2}{\varepsilon})$ or $\calQ_{\varepsilon} = \mathrm{Exp}(\frac{1}{\varepsilon})$.
On the other hand, the lower bound is proved in Theorem~\ref{thm:lower_det}. This means that our algorithm with the analyzed upper bound $O\left(\frac{\log K}{\varepsilon}\right)$ is optimal.
\end{proof}

%% file: neurips_2025.bbl
\begin{thebibliography}{28}
\providecommand{\natexlab}[1]{#1}
\providecommand{\url}[1]{\texttt{#1}}
\expandafter\ifx\csname urlstyle\endcsname\relax
  \providecommand{\doi}[1]{doi: #1}\else
  \providecommand{\doi}{doi: \begingroup \urlstyle{rm}\Url}\fi

\bibitem[Acharya et~al.(2021)Acharya, Sun, and
  Zhang]{acharya2021differentially}
J.~Acharya, Z.~Sun, and H.~Zhang.
\newblock Differentially private assouad, fano, and le cam.
\newblock In \emph{Algorithmic Learning Theory}, pages 48--78. PMLR, 2021.

\bibitem[Agarwal and Singh(2017)]{agarwal2017price}
N.~Agarwal and K.~Singh.
\newblock The price of differential privacy for online learning.
\newblock In \emph{International Conference on Machine Learning}, pages 32--40.
  PMLR, 2017.

\bibitem[Agarwal et~al.(2023)Agarwal, Kale, Singh, and
  Thakurta]{agarwal2023differentially}
N.~Agarwal, S.~Kale, K.~Singh, and A.~Thakurta.
\newblock Differentially private and lazy online convex optimization.
\newblock In \emph{The Thirty Sixth Annual Conference on Learning Theory},
  pages 4599--4632. PMLR, 2023.

\bibitem[Agarwal et~al.(2024)Agarwal, Kale, Singh, and
  Guha~Thakurta]{pmlr-v235-agarwal24d}
N.~Agarwal, S.~Kale, K.~Singh, and A.~Guha~Thakurta.
\newblock Improved differentially private and lazy online convex optimization:
  Lower regret without smoothness requirements.
\newblock In \emph{Proceedings of the 41st International Conference on Machine
  Learning}, pages 343--361. PMLR, 2024.

\bibitem[Arora et~al.(2012)Arora, Hazan, and Kale]{arora2012multiplicative}
S.~Arora, E.~Hazan, and S.~Kale.
\newblock The multiplicative weights update method: a meta-algorithm and
  applications.
\newblock \emph{Theory of computing}, 8\penalty0 (1):\penalty0 121--164, 2012.

\bibitem[Asi et~al.(2023{\natexlab{a}})Asi, Feldman, Koren, and
  Talwar]{asi2023near}
H.~Asi, V.~Feldman, T.~Koren, and K.~Talwar.
\newblock Near-optimal algorithms for private online optimization in the
  realizable regime.
\newblock In \emph{International Conference on Machine Learning}, pages
  1107--1120. PMLR, 2023{\natexlab{a}}.

\bibitem[Asi et~al.(2023{\natexlab{b}})Asi, Feldman, Koren, and
  Talwar]{asi2023private}
H.~Asi, V.~Feldman, T.~Koren, and K.~Talwar.
\newblock Private online prediction from experts: Separations and faster rates.
\newblock In \emph{The Thirty Sixth Annual Conference on Learning Theory},
  pages 674--699. PMLR, 2023{\natexlab{b}}.

\bibitem[Audibert and Bubeck(2009)]{audibert2009minimax}
J.-Y. Audibert and S.~Bubeck.
\newblock Minimax policies for adversarial and stochastic bandits.
\newblock In \emph{COLT}, pages 217--226, 2009.

\bibitem[Cesa-Bianchi and Lugosi(2006)]{cesa2006prediction}
N.~Cesa-Bianchi and G.~Lugosi.
\newblock \emph{Prediction, learning, and games}.
\newblock Cambridge university press, 2006.

\bibitem[Dwork et~al.(2010)Dwork, Naor, Pitassi, and
  Rothblum]{dwork2010differential}
C.~Dwork, M.~Naor, T.~Pitassi, and G.~N. Rothblum.
\newblock Differential privacy under continual observation.
\newblock In \emph{Proceedings of the forty-second ACM symposium on Theory of
  computing}, pages 715--724, 2010.

\bibitem[Dwork et~al.(2014)Dwork, Roth, et~al.]{dwork2014algorithmic}
C.~Dwork, A.~Roth, et~al.
\newblock The algorithmic foundations of differential privacy.
\newblock \emph{Foundations and Trends{\textregistered} in Theoretical Computer
  Science}, 9\penalty0 (3--4):\penalty0 211--407, 2014.

\bibitem[Freund and Schapire(1997)]{freund1997decision}
Y.~Freund and R.~E. Schapire.
\newblock A decision-theoretic generalization of on-line learning and an
  application to boosting.
\newblock \emph{Journal of computer and system sciences}, 55\penalty0
  (1):\penalty0 119--139, 1997.

\bibitem[Gumbel(1954)]{gumbel1954statistical}
E.~J. Gumbel.
\newblock Statistical theory of extreme valuse and some practical applications.
\newblock \emph{Nat. Bur. Standards Appl. Math. Ser. 33}, 1954.

\bibitem[Hu and Mehta(2024)]{hu2024open}
B.~Hu and N.~A. Mehta.
\newblock Open problem: Optimal rates for stochastic decision-theoretic online
  learning under differentially privacy.
\newblock In \emph{The Thirty Seventh Annual Conference on Learning Theory},
  pages 5330--5334. PMLR, 2024.

\bibitem[Hu et~al.(2021)Hu, Huang, and Mehta]{hu2021near}
B.~Hu, Z.~Huang, and N.~A. Mehta.
\newblock Near-optimal algorithms for private online learning in a stochastic
  environment.
\newblock \emph{arXiv preprint arXiv:2102.07929}, 2021.

\bibitem[Jain and Thakurta(2014)]{jain2014near}
P.~Jain and A.~G. Thakurta.
\newblock (near) dimension independent risk bounds for differentially private
  learning.
\newblock In \emph{International Conference on Machine Learning}, pages
  476--484. PMLR, 2014.

\bibitem[Jain et~al.(2012)Jain, Kothari, and Thakurta]{jain2012differentially}
P.~Jain, P.~Kothari, and A.~Thakurta.
\newblock Differentially private online learning.
\newblock In \emph{Conference on Learning Theory}, pages 24--1. JMLR Workshop
  and Conference Proceedings, 2012.

\bibitem[Kairouz et~al.(2021)Kairouz, McMahan, Song, Thakkar, Thakurta, and
  Xu]{kairouz2021practical}
P.~Kairouz, B.~McMahan, S.~Song, O.~Thakkar, A.~Thakurta, and Z.~Xu.
\newblock Practical and private (deep) learning without sampling or shuffling.
\newblock In \emph{International Conference on Machine Learning}, pages
  5213--5225. PMLR, 2021.

\bibitem[Kot{\l}owski(2018)]{kotlowski2018minimaxity}
W.~Kot{\l}owski.
\newblock On minimaxity of follow the leader strategy in the stochastic
  setting.
\newblock \emph{Theoretical Computer Science}, 742:\penalty0 50--65, 2018.

\bibitem[Lai and Robbins(1985)]{lai1985asymptotically}
T.~L. Lai and H.~Robbins.
\newblock Asymptotically efficient adaptive allocation rules.
\newblock \emph{Advances in applied mathematics}, 6\penalty0 (1):\penalty0
  4--22, 1985.

\bibitem[McKenna and Sheldon(2020)]{mckenna2020permute}
R.~McKenna and D.~R. Sheldon.
\newblock Permute-and-flip: A new mechanism for differentially private
  selection.
\newblock \emph{Advances in Neural Information Processing Systems},
  33:\penalty0 193--203, 2020.

\bibitem[McSherry and Talwar(2007)]{mcsherry2007mechanism}
F.~McSherry and K.~Talwar.
\newblock Mechanism design via differential privacy.
\newblock In \emph{48th Annual IEEE Symposium on Foundations of Computer
  Science (FOCS'07)}, pages 94--103. IEEE, 2007.

\bibitem[Mourtada and Ga{\"\i}ffas(2019)]{mourtada2019optimality}
J.~Mourtada and S.~Ga{\"\i}ffas.
\newblock On the optimality of the hedge algorithm in the stochastic regime.
\newblock \emph{Journal of Machine Learning Research}, 20\penalty0
  (83):\penalty0 1--28, 2019.

\bibitem[Qiao et~al.(2021)Qiao, Su, and Zhang]{qiao2021oneshot}
G.~Qiao, W.~Su, and L.~Zhang.
\newblock Oneshot differentially private top-k selection.
\newblock In \emph{International Conference on Machine Learning}, pages
  8672--8681. PMLR, 2021.

\bibitem[Sajed and Sheffet(2019)]{sajed2019optimal}
T.~Sajed and O.~Sheffet.
\newblock An optimal private stochastic-mab algorithm based on optimal private
  stopping rule.
\newblock In \emph{International Conference on Machine Learning}, pages
  5579--5588. PMLR, 2019.

\bibitem[Smith and Thakurta(2013)]{guha2013nearly}
A.~Smith and A.~G. Thakurta.
\newblock (nearly) optimal algorithms for private online learning in
  full-information and bandit settings.
\newblock \emph{Advances in Neural Information Processing Systems}, 26, 2013.

\bibitem[Tossou and Dimitrakakis(2017)]{tossou2017achieving}
A.~Tossou and C.~Dimitrakakis.
\newblock Achieving privacy in the adversarial multi-armed bandit.
\newblock In \emph{Proceedings of the AAAI Conference on Artificial
  Intelligence}, volume~31, 2017.

\bibitem[Wadsworth and Bryan(1960)]{wadsworth1960introduction}
G.~P. Wadsworth and J.~G. Bryan.
\newblock \emph{Introduction to probability and random variables}, volume~7.
\newblock McGraw-Hill New York, 1960.

\end{thebibliography}
